\documentclass[sigconf,screen]{acmart}
\usepackage{url}
\usepackage{amsthm}
\usepackage{afterpage}
\usepackage{subcaption}
\usepackage[normalem]{ulem}
\useunder{\uline}{\ul}{}
\usepackage{graphicx}
\usepackage{textcomp}
\usepackage{xcolor}
\usepackage{balance} 
\usepackage{bbm}
\usepackage{amsmath,amsfonts}
\usepackage{booktabs}
\usepackage{tcolorbox}

\usepackage{epstopdf}
\usepackage{array}
\usepackage{booktabs}
\usepackage{multicol}
\usepackage{color}
\usepackage{xcolor}
\usepackage{colortbl}
\usepackage{xspace}
\usepackage{multirow}
\usepackage{pifont}
\usepackage{enumitem}
\usepackage{bbding}
\usepackage{hyperref}
\usepackage{makecell}
\usepackage{microtype}
\usepackage{fontawesome}
\usepackage{caption,setspace}
\usepackage[utf8]{inputenc}
\usepackage{utfsym}
\usepackage{diagbox}
\usepackage{makecell}
\newcommand{\thickhline}{\noalign{\hrule height 0.3pt}} 

\DeclareMathOperator*{\argmin}{arg\,min}

\AtBeginDocument{%
 }
\acmSubmissionID{7509}



\usepackage[linesnumbered,ruled,vlined]{algorithm2e}

\newcommand{\ssymbol}[1]{^{\@fnsymbol{#1}}}

\SetKwInput{KwInput}{Input}                
\SetKwInput{KwOutput}{Output}              
\SetKwRepeat{Do}{do}{while}

\definecolor{darkred}{rgb}{0.55, 0.0, 0.0}
\definecolor{lightred}{rgb}{0.94, 0.5, 0.5}
\definecolor{rosybrown}{rgb}{0.74, 0.56, 0.56}
\definecolor{darkgreen}{rgb}{0.0, 0.39, 0.0}
\definecolor{skyblue}{rgb}{0.56, 0.93, 0.56}
\definecolor{royalblue}{rgb}{0.0, 0.0, 0.55}
\definecolor{lightblue}{rgb}{0.68, 0.85, 0.9}
\definecolor{skyblue}{rgb}{0.53, 0.81, 0.92}

\definecolor{gray}{rgb}{0.5, 0.5, 0.5}
\definecolor{forestgreen}{rgb}{0.13, 0.55, 0.13}
\definecolor{slateblue}{rgb}{0.42, 0.35, 0.80}
\definecolor{darkgray}{rgb}{0.33, 0.33, 0.33}
\definecolor{royalblue}{rgb}{0.25, 0.41, 0.88}

\SetCommentSty{mycommfont}

\usepackage{ulem}

\begin{document}

\title{TMTE: Effective Multimodal Graph Learning with\\ Task-aware Modality and Topology Co-evolution}

\author{Yinlin Zhu}
\affiliation{%
\institution{Sun Yat-sen University}
\city{Guangzhou}
\country{China}}
\email{zhuylin27@mail2.sysu.edu.cn}

\author{Xunkai Li}
\affiliation{%
\institution{Beijing Institute of Technology}
\city{Beijing}
\country{China}}
\email{cs.xunkai.li@gmail.com}

\author{Di Wu}
\authornote{Corresponding author}
\affiliation{
  \institution{Sun Yat-sen University}
  \city{Guangzhou}
  \country{China}
}
\email{wudi27@mail.sysu.edu.cn}

\author{Wang Luo}
\affiliation{%
\institution{Sun Yat-sen University}
\city{Guangzhou}
\country{China}}
\email{luow69@mail2.sysu.edu.cn}

\author{Miao Hu}
\affiliation{%
\institution{Sun Yat-sen University}
\city{Guangzhou}
\country{China}}
\email{humiao5@mail.sysu.edu.cn}

\author{Guocong Quan}
\affiliation{%
\institution{Sun Yat-sen University}
\city{Guangzhou}
\country{China}}
\email{quangc@mail.sysu.edu.cn}

\renewcommand{\shortauthors}{Yinlin Zhu, Xunkai Li, Di Wu, Wang Luo, Miao Hu, Di Wu}

\begin{abstract}
Multimodal-attributed graphs (MAGs) are a fundamental data structure for multimodal graph learning (MGL), enabling both graph-centric and modality-centric tasks. However, our empirical analysis reveals inherent topology quality limitations in real-world MAGs, including noisy interactions, missing connections, and task-agnostic relational structures. A single graph derived from generic relationships is therefore unlikely to be universally optimal for diverse downstream tasks. To address this challenge, we propose \underline{\textbf{T}}ask-aware \underline{\textbf{M}}odality and \underline{\textbf{T}}opology co-\underline{\textbf{E}}volution (TMTE), a novel MGL framework that jointly and iteratively optimizes graph topology and multimodal representations toward the target task. TMTE is motivated by the bidirectional coupling between modality and topology: multimodal attributes induce relational structures, while graph topology shapes modality representations. Concretely, TMTE casts topology evolution as multi-perspective metric learning over modality embeddings with an anchor-based approximation, and formulates modality evolution as smoothness-regularized fusion with cross-modal alignment, yielding a closed-loop task-aware co-evolution process. Extensive experiments on 9 MAG datasets and 1 non-graph multimodal dataset across 6 graph-centric and modality-centric tasks show that TMTE consistently achieves state-of-the-art performance. Our code is available at \href{https://anonymous.4open.science/r/TMTE-1873}{this repository}.
\end{abstract}

 \maketitle

\section{Introduction}
\label{sec: introduction}

    Multimodal-attributed graphs (MAGs) represent entities as nodes with multimodal attributes (e.g., images and texts) and encode their relational dependencies as edges, which have been widely adopted in various real-world applications, such as finance~\cite{mag_app_finance}, biochemistry~\cite{mag_app_bio}, and recommendation systems~\cite{lgmrec}. Building upon this expressive data structure, multimodal graph learning (MGL) has attracted growing attention in recent years. On the one hand, MGL enhances graph-centric tasks (e.g., node classification~\cite{unigraph2} and link prediction~\cite{lion}) by leveraging rich multimodal attributes; on the other hand, it benefits modality-centric tasks (e.g., graph-to-text~\cite{graphgpt_o} and graph-to-image~\cite{instructg2i} generation) by exploiting graph topology for contextual information~\cite{mai2023multimodal}.

    Despite their notable advances, most existing MGL methods rely on a \textit{perfect topology assumption}, where the real-world graphs are presumed to be sufficiently well-structured for downstream tasks~\cite{idgl}. However, this assumption often does not hold in practice, which can be summarized from three perspectives: (1) \textbf{Noisy Interactions.} Due to inherent modality noise in the MAGs or imperfections in the collection process, the graph topology may contain spurious connections~\cite{dgf}. (2) \textbf{Missing Interactions.} The collected MAGs may overlook potentially informative connections. Moreover, in some practical applications with sequential entities such as audiovisual speech recognition~\cite{mm_temporal_seq} or isolated multimodal entities like image-text matching~\cite{clip}, no explicit graph priors are available to characterize the inter-entity relationships. (3) \textbf{Task-agnostic Interactions.} Both intuitive reasoning and our empirical study (Sec.~\ref{sec: empirical investigation}) suggest that a single graph structure derived from task-agnostic relationships (e.g., co-occurrence patterns) is unlikely to be universally optimal for diverse downstream tasks, since different tasks often benefit from distinct contextual signals. For example, classification tasks emphasize semantic category consistency, while generation tasks may require semantic diversity. Notably, existing methods~\cite{dgf} prune spurious links based on low cross-modal similarity, thereby targeting only noisy interactions while overlooking the other challenges. Consequently, existing MGL methods inevitably depend on multimodal contexts induced by suboptimal topologies, which compromises downstream task performance.

    \textit{How can we adaptively learn graph structures that faithfully capture inter-entity relationships while simultaneously providing task-specific contextual cues for diverse downstream tasks, thereby enabling effective MGL?}
    In this work, we address this question by proposing a novel MGL framework, termed \underline{\textbf{T}}ask-aware \underline{\textbf{M}}odality and \underline{\textbf{T}}opology co-\underline{\textbf{E}}volution (TMTE). TMTE jointly and iteratively optimizes graph topology and multimodal representations toward the target task.
    Our key insight lies in the intrinsic coupling between modality and topology in MGL:
    (1) \textit{\textbf{Modality $\rightarrow$ Topology.}} Multimodal descriptions of entities naturally induce relational structures; for instance, user attributes shape social connections, and molecular properties govern bonding patterns.
    (2) \textit{\textbf{Topology $\rightarrow$ Modality.}} Conversely, graph topology influences modality representations via message passing or cross-modal alignment, thereby shaping the latent node embeddings used for downstream tasks.
    Concretely, TMTE casts topology evolution as multi-perspective metric learning over modality embeddings with an anchor-based approximation, and formulates modality evolution as smoothness-regularized fusion with cross-modal alignment. The refined topology provides improved contextual signals tailored to each node with respect to the downstream objective and induces enhanced modality representations through message passing or graph contrastive alignment. This procedure iterates across training epochs and forms a task-aware modality-topology co-evolution process, converging when the learned structure sufficiently approximates a topology optimized for the target task.

    \textbf{Our Contributions:} 
    (1) \textbf{In-depth Investigation.} We provide an in-depth empirical analysis to demonstrate the inherent topology quality limitations of MAGs, as different downstream tasks often require distinct graph topologies to provide appropriate contextual signals.
    (2) \textbf{Novel Method.} We propose TMTE, a novel MGL framework that jointly and evolutionarily optimizes the graph topology and modality attributes toward the downstream task.
    (3) \textbf{State-of-the-art Performance.} Extensive experimental results on 9 MAG datasets and 1 non-graph multimodal dataset demonstrate that TMTE consistently outperforms the state-of-the-art baselines on 3 graph-centric and 3 modality-centric downstream tasks.

\section{Problem Formulation}

A MAG is denoted as $\mathcal{G} = (\mathcal{V}, \mathcal{E}, \{\mathbf{X}^{(m)}\}_{m \in \mathcal{M}})$, 
where $\mathcal{V}$ is the set of nodes with $N = |\mathcal{V}|$, $\mathcal{E}$ is the set of edges, and $\mathcal{M}$ indexes the available modalities (e.g., text and images). 
For each node $v_i$ under modality $m$, we associate a modality-specific feature vector $\mathbf{x}_i^{(m)} \in \mathbb{R}^{d_m}$, obtained by encoding the raw modality data with a pre-trained modality encoder (e.g., Sentence-BERT~\cite{sentence_bert} and T5~\cite{t5} for text; ViT~\cite{vit} and DINOv2~\cite{dinov2} for images). 
For each modality $m$, the embeddings of all nodes are organized into a feature matrix $\mathbf{X}^{(m)} \in \mathbb{R}^{N \times d_m}$. 
All modalities share a common relational structure represented by the adjacency matrix $\mathbf{A} \in \mathbb{R}^{N \times N}$, with the corresponding degree matrix $\mathbf{D}$, where $\mathbf{D}_{ii} = \sum_j \mathbf{A}_{ij}$ and $\mathbf{D}_{ij} = 0$ for $i \neq j$. To capture structural smoothness, the symmetrically normalized adjacency matrix and Laplacian of $\mathcal{G}$ are represented by $\tilde{\mathbf{A}}=\mathbf{D}^{-1/2} \mathbf{A} \mathbf{D}^{-1/2}$ and $\mathbf{L} = \mathbf{I} - \tilde{\mathbf{A}}$, respectively. 

In this paper, we consider three \textbf{Graph-centric Tasks}, including (1) \textit{Node Classification}: assigning category labels to unlabeled nodes;
(2) \textit{Link Prediction}: estimating whether an edge $(u,v)$ should exist in $\mathcal{E}$; (3) \textit{Node Clustering}: grouping nodes into clusters by applying standard clustering algorithms to the learned node embeddings; and three \textbf{Modality-centric Tasks}, including (1) \textit{Modality Retrieval}: given a query from one modality (e.g., text), retrieving its corresponding representation in another modality;
(2) \textit{Graph-to-Text Generation (G2Text)}: producing textual outputs conditioned on a target node along with task instructions and its graph neighborhood;
(3) \textit{Graph-to-Image Generation (G2Image)}: synthesizing images using diffusion models conditioned on the node’s textual descriptions and the associated graph context. Further discussion of these tasks is presented in Appendix~\ref{appendix: more experimental setups}.

\section{Empirical Investigation}
\label{sec: empirical investigation}

In this section, we conduct a comprehensive analysis of the topology quality limitations of MAGs (as discussed in Sec.~\ref{sec: introduction}). Our goal is to systematically examine the following research questions, including \textbf{Q1}: Are the original MAGs optimal for downstream tasks? \textbf{Q2}: Do different downstream tasks require distinct contextual multimodal semantics induced by different graph topologies? 
and \textbf{Q3}: Can TMTE learn task-adaptive graph topologies that better align with downstream objectives?

\vspace{+0.1cm}
\noindent \textbf{Datasets, Methods, and Tasks.} The controlled study is conducted to compare the proposed TMTE and MM-GCN~\cite{mmgcn} on the Toys~\cite{ni2019justifying} dataset. We consider two representative downstream tasks, including a graph-centric task (node classification) and a modality-centric task (G2Image). For each task, we evaluate three different topology configurations. All other experimental settings are kept consistent across methods. 
Detailed dataset, baselines, and pipelines of both the node classification and G2Image tasks are presented in Appendices~\ref{appendix: dataset details},~\ref{appendix: baseline details} and~\ref{appendix: more experimental setups}, respectively.

\begin{figure}[htb]
    \centering
    \includegraphics[width=0.49\textwidth]{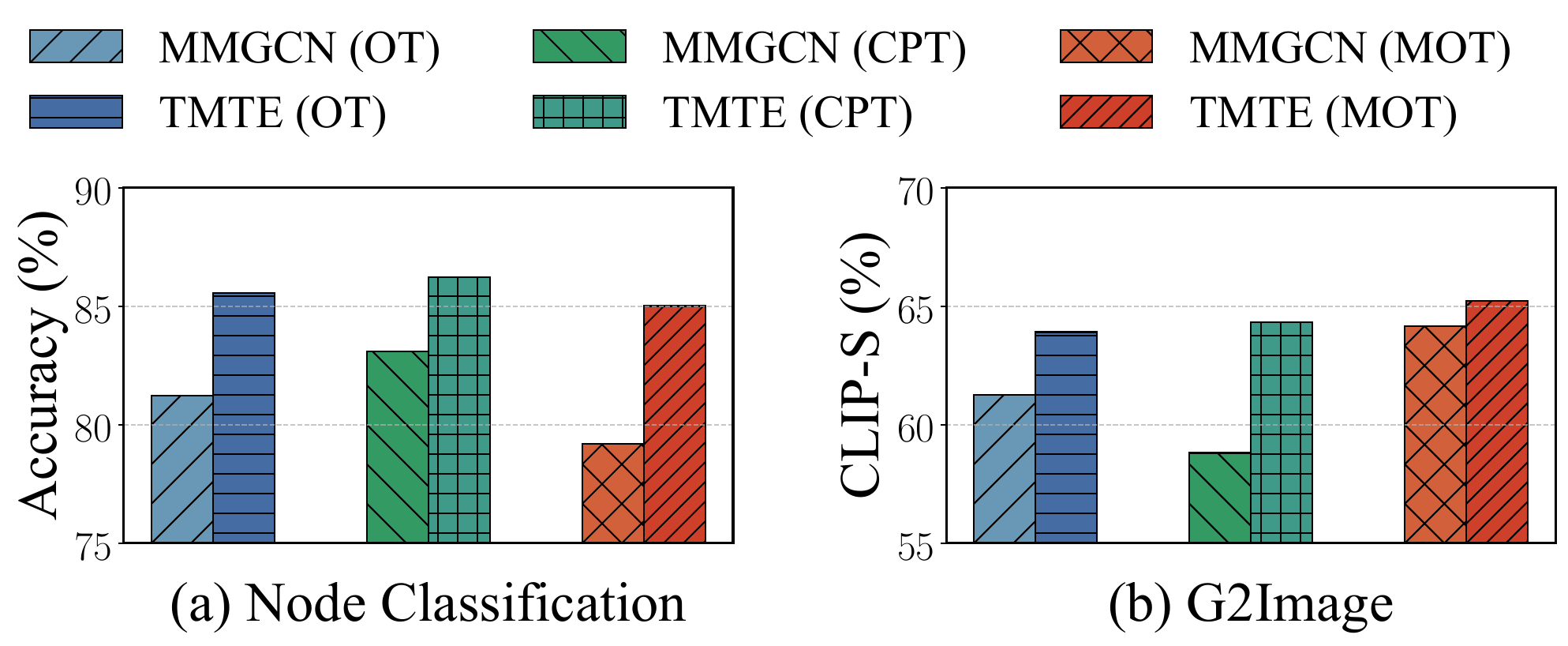}
    \caption{\textbf{Experimental results of our empirical study}. We compare the proposed TMTE and MM-GCN (representative MGL baseline) on the Toys dataset. (a) Node classification performance under three topology settings. (b) G2Image performance under three topology settings. All results are presented as percentages.}
    \label{fig: empirical study}
\end{figure}

\vspace{+0.1cm}
\noindent \textbf{Topology Variants.} We construct the following three graph structures, thereby providing different local context semantics for each node: (1) \textit{Original Topology.} The raw graph structure provided by the Toys dataset. (2) \textit{Cross-category Pruned Topology.} The original graph with cross-category edges removed. This operation aims to suppress potentially noisy context arising from connections between semantically inconsistent categories, thereby improving performance on node classification tasks.
(3) \textit{Modality-Optimized Topology.} The original graph is refined according to image-text cross-modal similarity. More details are presented in Appendix~\ref{appendix: more experimental setups}.

\vspace{+0.1cm}
\noindent \textbf{The original topology is not optimal (Answer for Q1).}
We first observe that the original graph structure does not yield the best performance for either task. For the \textit{node classification} (Fig.~\ref{fig: empirical study} (a)), the cross-category pruned topology (\textbf{i.e., CPT}) consistently outperforms the original topology (\textbf{i.e., OT}). This suggests that the original graph contains noisy inter-category connections that dilute discriminative category-specific contexts required for classification. Similarly, for the \textit{G2Image} task (Fig.~\ref{fig: empirical study} (b)), the modality-optimized topology (\textbf{i.e., MOT}) surpasses the original topology, indicating that the original graph fails to provide sufficiently aligned multimodal contexts for high-quality image generation.

\vspace{+0.1cm}
\noindent  \textbf{Different tasks prefer different topologies (Answer for Q2).}
More importantly, the optimal topology varies across tasks. For the \textit{node classification} task, the cross-category pruned topology achieves higher performance than the modality-optimized topology for the MMGCN and TMTE methods. This implies that classification primarily benefits from cleaner, label-consistent structural neighborhoods. In contrast, for the \textit{G2Image} task, the modality-optimized topology significantly outperforms the cross-category pruned topology. This demonstrates that modality-centric generation tasks rely more heavily on fine-grained cross-modal semantic alignment than strict category homophily. These results indicate that a single task-agnostic topology is unlikely to be universally optimal, as different downstream objectives require distinct topological priors to provide appropriate multimodal contextual signals. Therefore, naively conducting MGL on these topologies yields sub-optimal performance.

\vspace{+0.1cm}
\noindent  \textbf{TMTE can learn task-adaptive topologies (Answer for Q3).}
Finally, we analyze the performance of TMTE across different topology configurations. 
Figure~\ref{fig: empirical study} reveals the following two trends:
(1) Under all three topology settings, TMTE consistently outperforms MM-GCN on both node classification and G2Image tasks. 
(2) The performance variation of TMTE across different input topologies is relatively smaller compared with MM-GCN. These results indicate that TMTE is less sensitive to the quality and configuration of the input graph structure, while MM-GCN exhibits more noticeable performance fluctuations when the topology changes. 
This improved robustness indicates that the task-aware topology and modality co-evolution mechanism can optimize the structural dependencies during training rather than relying solely on a fixed input graph.

\section{Methodology}
\label{sec: methodology}

In this section, we introduce TMTE, a novel MGL framework that jointly and evolutionarily optimizes both graph topology and modality representations toward downstream tasks. We first present an overview of TMTE in Fig.~\ref{fig: framework}. We then elaborate on the key components in detail. Specifically, we describe the topology evolution process in Sec.~\ref{sec: topology evolution}, the modality evolution mechanism in Sec.~\ref{sec: modality evolution}, and finally, the task-aware co-evolution strategy that integrates both processes in Sec.~\ref{sec: task co-evo}.

\subsection{Topology Evolution from Original Modality Feature Space}
\label{sec: topology evolution}

As discussed in Sec.~\ref{sec: introduction}, the topology of the original MAG may contain noisy connections while omitting informative ones. To address this issue, TMTE first encourages the topology to preserve intrinsic inter-entity relationships within the original modality feature space, serving as the starting point for topology evolution.

\vspace{+0.1cm}
\noindent \textbf{Multimodal and Multi-perspective Similarity Metric Learning.} We formulate topology evolution as a similarity metric learning problem driven by multimodal semantics from multiple perspectives. Specifically, given a MAG $\mathcal{G} = (\mathcal{V}, \mathcal{E}, {\mathbf{X}^{(m)}}_{m \in \mathcal{M}})$, we first compute the fused representation as $\bar{\mathbf{X}} = \frac{1}{|\mathcal{M}|}\sum_{m\in\mathcal{M}}\mathbf{X}^{(m)}$ and extend the modality index set as $\bar{\mathcal{M}} = \mathcal{M} \cup \{|\mathcal{M}|+1\}$. For each $m \in \bar{\mathcal{M}}$, we introduce $K$ learnable parameter vectors $\big[\mathbf{w}^{(m,1)}, \ldots, \mathbf{w}^{(m,K)}\big]$ to capture pairwise similarities from distinct perspectives using weighted cosine similarity. The similarity between nodes $v_i$ and $v_j$ under modality $m$ is computed by aggregating these perspectives:
\begin{equation}
\label{eq: sim channel}
a_{ij}^{(m,p)} =
\mathrm{cos}\!\left(
\mathbf{w}^{(m,p)} \odot \mathbf{x}_i^{(m)},
\mathbf{w}^{(m,p)} \odot \mathbf{x}_j^{(m)}
\right),
\qquad
s_{ij}^{(m)} = \frac{1}{K}\sum_{p=1}^K a_{ij}^{(m,p)},
\end{equation}
where $\odot$ denotes the Hadamard product. Finally, the overall pairwise similarity between nodes $v_i$ and $v_j$ is obtained via a convex combination of $|\mathcal{M}|+1$ modality-specific similarities: 
\begin{equation}
\label{eq: sim combine}
\mathbf{A}^{E_1}_{ij} = \sum_{m \in \bar{\mathcal{M}}} \mathrm{softmax}(\mathbf{\beta})_m \ s_{ij}^{(m)},
\end{equation}
where $\mathbf{\beta}\in\mathbb{R}^{|\mathcal{M}|+1}$ is jointly optimized with the model parameters to adaptively balance these modality-specific semantics. We define $\mathbf{A}^{E_1}$ as the starting point for topology evolution. Since $\mathbf{A}^{E_1}_{ij} \in [-1,1]$, we further apply a non-negative threshold $\epsilon$ by setting elements below $\epsilon$ to zero, forming a symmetric, sparse, and non-negative weighted adjacency matrix $\mathbf{A}^{E_1}\in \mathbb{R}^{|\mathcal{V}|\times|\mathcal{V}|}$.

\vspace{+0.1cm}
\noindent  \textbf{Node-anchor Affinity Matrix for Scalability.} However, computing full pairwise similarities would require $\mathcal{O}(|\bar{\mathcal{M}}|\cdot |\mathcal{V}|^2)$ time and memory and is prohibitive for large MAGs. Inspired by studies in large-scale graph learning~\cite{large_scale_graph_learning}, we adopt an anchor-based scalable formulation. Specifically, let $\mathcal{U} \subseteq \mathcal{V}$ denote an anchor set with $|\mathcal{U}|$ randomly sampled nodes. We compute a node-anchor affinity matrix $\mathbf{R}^{E_1} \in \mathbb{R}^{|\mathcal{V}| \times |\mathcal{U}|}$ using the same weighted cosine scheme as in Eqs.~\eqref{eq: sim channel} and~\eqref{eq: sim combine}. This preserves $\mathcal{O}(|\bar{\mathcal{M}}| \cdot|\mathcal{V}| \cdot |\mathcal{U}|)$ complexity, where $|\mathcal{U}| \ll |\mathcal{V}|$ holds. $\mathbf{R}^{E_1}$ serves as the starting point of topology evolution across a bipartite graph $\mathcal{B}$ between $\mathcal{V}$ and $\mathcal{U}$, which can further induce the global topology evolution term $\mathbf{A}^{E_1}\in \mathbb{R}^{|\mathcal{V}| \times |\mathcal{V}|}$ that refines pairwise relations $\mathcal{V}\times \mathcal{V}$, formulated as:
\begin{equation}
    \label{eq: recover}
    \mathbf{A}^{E_1}={\mathbf{\Delta}^{E_1}}^{-1} \mathbf{R}^{E_1} {\mathbf{\Lambda}^{E_1}}^{-1} {\mathbf{R}^{E_1}}^\top, 
\end{equation}
where $\mathbf{\Lambda}^{E_1} \in \mathbb{R}^{|\mathcal{U}| \times |\mathcal{U}|}$ ($\mathbf{\Lambda}^{E_1}_{jj}=\sum_{i=1}^{|\mathcal{N}|}{\mathbf{R}^{E_1}_{ij}}$) and $\mathbf{\Delta}^{E_1} \in \mathbb{R}^{|\mathcal{V}| \times |\mathcal{V}|}$ ($\mathbf{\Delta}^{E_1}_{ii}=\sum_{j=1}^{|\mathcal{U}|}{\mathbf{R}^{E_1}_{ij}}$) are both diagonal matrices. Importantly, Eq.~\eqref{eq: recover} does not need to be computed explicitly. Instead, the contextual semantics are diffused over the evolved topology through a two-stage smoothing operation induced by $\mathcal{B}$, which is detailed in Sec.~\ref{sec: modality evolution}.

To employ structural priors encoded in the initial topology, the learned pairwise relations $\mathbf{A}^{E_1}$ serve as residual structural refinements, correcting rather than replacing the original graph. For intuition, the topology evolution essentially leads to the symmetrically normalized adjacency matrix $\mathbf{Q}^{E_1}$:
\begin{equation}
\label{eq: combine_adj_norm_t}
\begin{aligned}
\mathbf{Q}^{E_1} = \lambda \tilde{\mathbf{A}} + (1 - \lambda) \mathbf{A}^{E_1},
\end{aligned}
\end{equation}
where $\tilde{\mathbf{A}}$ denotes the symmetrically normalized adjacency matrix of the original topology, and $\lambda \in (0,1)$ is a trade-off parameter to balance the original topology and evolved ones.

\begin{figure*}[htb]
 \centering
  \includegraphics[width=0.95\textwidth]{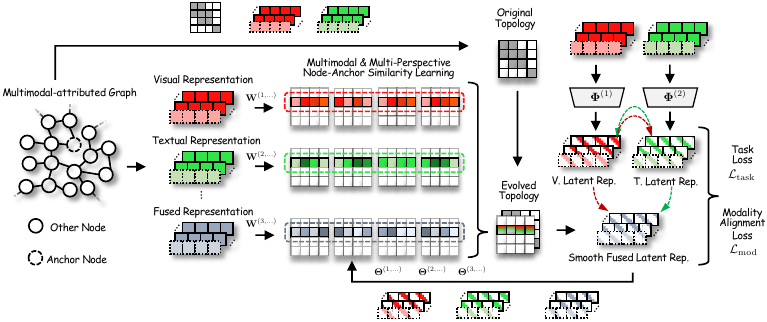}
  \caption{Overview of TMTE, which jointly and evolutionarily optimizes the topology and modality toward the downstream task.}
\label{fig: framework}
\end{figure*}

\subsection{Modality Evolution from Evolved Topology}
\label{sec: modality evolution}

After the topology evolution, we further perform modality evolution on the refined structure. Motivated by prior work in graph signal processing (GSP)~\cite{gsp_lap_smooth}, a well-learned node representation should be smooth over the graph topology. Accordingly, we model modality evolution as a two-step process: (1) learning node representations that capture all modality information while being smooth with respect to the graph, and (2) aligning each modality with the learned node embeddings.

\vspace{+0.1cm}
\noindent \textbf{Learning Smooth Fused Representations.} Formally, for the $m$-th modality, we project its raw feature space $\mathbf{X}^{(m)}$ into a latent space $\mathbf{H}^{(m)}$ via a linear transformation $\mathbf{\Phi}^{(m)}$:
\begin{equation}
    \label{eq: modality embedding}
    \mathbf{H}^{(m)} = \mathbf{X}^{(m)} \mathbf{\Phi}^{(m)}.
\end{equation}
The initial fused node representation is then obtained by averaging across modalities, i.e., $\bar{\mathbf{H}}=\frac{1}{|\mathcal{M}|}\sum_{m\in\mathcal{M}}\mathbf{H}^{(m)}$.
As our goal is to enforce representation smoothness over the evolved topology $\mathbf{Q}^{E_1}$, we further minimize the following optimization objective $f(\hat{\mathbf{H}})$:
\begin{equation}
    \label{eq: smooth}
        \min_{\hat{\mathbf{H}}\in\mathbb{R}^{|V|\times d}}f(\hat{\mathbf{H}})=||\hat{\mathbf{H}} - \bar{\mathbf{H}}||_F^2 + \alpha \cdot \text{tr}(\hat{\mathbf{H}}^\top(\mathbf{I}-\mathbf{Q}^{E_1})\hat{\mathbf{H}}),
\end{equation}
where $\alpha>0$ controls the smoothness strength, and the Laplacian regularizer $\mathbf{I}-\mathbf{Q}^{E_1}$ encourages neighboring nodes to have similar embeddings. The fidelity term $||\hat{\mathbf{H}} - \bar{\mathbf{H}}||_F^2$ ensures the learned embeddings $\hat{\mathbf{H}}$ remain close to the initial fused representation $\bar{\mathbf{H}}$.

By setting $\frac{\partial f(\hat{\mathbf{H}})}{\partial \hat{\mathbf{H}}}=0$ and further obtaining a $T$-step truncation closed-form solution, we yield the power series approximation of $\hat{\mathbf{H}}$, whose full derivation is provided in Theorem~\ref{thm: smooth fused rep}:
\begin{equation}
    \label{eq: approximated}
    \hat{\mathbf{H}}
    =
    \frac{1}{\alpha+1}
    \sum_{t=0}^{T}
    \left(
    \frac{\alpha\lambda}{\alpha+1}\tilde{\mathbf{A}}
    +
    \frac{\alpha(1-\lambda)}{\alpha+1}\mathbf{A}^{E_1}
    \right)^t
    \bar{\mathbf{H}}.
\end{equation}
Importantly, TMTE does not require explicitly constructing $\mathbf{A}^{E_1}$. 
Instead, Eq.~\eqref{eq: approximated} can be efficiently computed through recursive multiplications. For example, the multiplication $\mathbf{A}^{E_1}\bar{\mathbf{H}}$ can be decomposed as a two-stage smoothing operator (node-to-anchor-to-node transformation):
\begin{equation}
    \mathbf{Z}_{\mathcal{U}}
    =
    {\mathbf{\Lambda}^{E_1}}^{-1}
    {\mathbf{R}^{E_1}}^\top
    \bar{\mathbf{H}},
    \quad
    \mathbf{Z}_{\mathcal{V}}
    =
    {\mathbf{\Delta}^{E_1}}^{-1}
    \mathbf{R}^{E_1}
    \mathbf{Z}_{\mathcal{U}},
\end{equation}
where $\mathbf{Z}_{\mathcal{V}} = \mathbf{A}^{E_1}\bar{\mathbf{H}}$. Proofs and pseudocode refer to Theorem~\ref{thm: recursive power expansion} and Appendix~\ref{appendix: pseudo code}, respectively.

\vspace{+0.1cm}
\noindent  \textbf{Modality Alignment.} 
To induce modality evolution while maintaining cross-modality consistency, 
we align each modality embedding $\mathbf{H}^{(m)}$ with other modalities 
and the fused representation $\hat{\mathbf{H}}$ via a unified contrastive loss, which can be formulated as follows:
\begin{equation} \mathcal{L}_{\text{mod}} = - \sum_{m=1}^{|\mathcal{M}|} \sum_{n=1, n\neq m}^{|\mathcal{M}|+1} \sum_{u} \log \frac{ \exp\!\big(\mathrm{cos}(\mathbf{h}^{(m)}_u, \operatorname{csg}(\mathbf{h}^{(n)}_u))/\tau\big) }{ \sum_{v} \exp\!\big(\mathrm{cos}(\mathbf{h}^{(m)}_u, \operatorname{csg}(\mathbf{h}^{(n)}_v))/\tau\big) }, \end{equation}
where $\mathbf{h}^{(|\mathcal{M}|+1)}_u$ denotes the fused node embedding $\hat{\mathbf{h}}_u$, 
$\operatorname{csg}(\cdot)$ is the conditioned stop-gradient (i.e., applied only to the smooth fused representation), 
and $\tau$ is the temperature parameter.

\subsection{Task-aware Modality and Topology Co-evolution}
\label{sec: task co-evo}

\textbf{Optimization toward Downstream Tasks.}
The framework is task-agnostic and can be optimized for arbitrary downstream objectives with corresponding task-specific heads.
For \textit{graph-centric tasks} (e.g., node classification, link prediction, node clustering), we use the smooth fused representation $\hat{\mathbf{H}}$ in Eq.~\eqref{eq: approximated} as the topology-aware node embedding.
For \textit{modality-centric tasks} (e.g., modality retrieval, G2Text, G2Image), we adopt the modality-specific latent embeddings $\mathbf{H}^{(m)}$ in Eq.~\eqref{eq: modality embedding}, which preserve modality characteristics while being aligned via $\mathcal{L}_{\text{mod}}$.

We denote the unified downstream objective as $\mathcal{L}_{\text{task}}$, which is instantiated according to the specific task.
The overall optimization objective of TMTE is formulated as follows:
\begin{equation}
\label{eq: loss_tmte}
\mathcal{L}=
\mathcal{L}_{\text{mod}} + \eta\, \mathcal{L}_{\text{task}},
\end{equation}
where $\eta$ balances modality alignment and task objective (More details are presented in Appendix~\ref{appendix: more experimental setups}).

\vspace{+0.1cm}
\noindent \textbf{Modality and Topology Co-evolution.} Topology and modality evolution in TMTE form a closed-loop co-evolution process.
The evolved topology $\mathbf{Q}^{E_t}$ guides modality evolution via the smoothness constraints in Eqs.~\eqref{eq: smooth} and \eqref{eq: approximated}, while the updated representations $\hat{\mathbf{H}}$ and $\mathbf{H}^{(m)}$ provide refined similarity signals for topology reconstruction. Thus, representation learning and topology refinement mutually enhance each other.

Specifically, at the $t$-th round, we update topology using the smooth fused embedding in Eq.~\eqref{eq: approximated} and modality embeddings in Eq.~\eqref{eq: modality embedding}. Similar to Eqs.~\eqref{eq: sim channel} and \eqref{eq: sim combine}, multi-channel similarities are aggregated via $[\mathbf{\theta}^{(m,k)}]_{m\leqslant |\mathcal{M}|+1, k\leqslant K}$ and $\gamma$ to obtain $\mathbf{R}^{(E_t)}$, yielding an implicitly symmetrically normalized adjacency matrix $\mathbf{Q}^{E_t}$, which can be calculated as follows:
\begin{equation}
\label{eq: combine_adj_norm_t_repeat}
\mathbf{Q}^{E_t}=
\lambda \tilde{\mathbf{A}} + (1 - \lambda)\mathbf{A}^{E_t}
\end{equation}
Modality evolution then replaces $\mathbf{Q}^{E_1}$ with $\mathbf{Q}^{E_t}$ in Eqs.~\eqref{eq: smooth} and \eqref{eq: approximated}, producing $\hat{\mathbf{H}}^{(t)}$ and ${\mathbf{H}^{(m,t)}}$.

We repeat this process for at most $T$ rounds, and stop early when topology changes are smaller than a threshold $\delta$, which is formulated as follows:
\begin{equation}
\frac{|\mathbf{R}^{(E_t)} - \mathbf{R}^{(E_{t-1})}|_F^2}{|\mathbf{R}^{(E_t)}|_F^2}
\leqslant
\delta
,
\end{equation}
This iterative co-evolution progressively refines topology and modality representations in a task-aware and mutually adaptive manner. The overall procedure of TMTE is presented in Appendix~\ref{appendix: pseudo code}.


\begin{table*}[htbp]
\setlength{\abovecaptionskip}{0.2cm}
\setlength{\belowcaptionskip}{-0.2cm}
\centering
\caption{Performance comparison on three \textbf{graph-centric} downstream tasks. The best, second best and third best results are highlighted in \textcolor{darkred}{\textbf{red}}, \textcolor{royalblue}{\textbf{blue}} and \textcolor{orange}{\textbf{orange}}, respectively.}
\label{tab: graph_centric}
\footnotesize 
\renewcommand{\arraystretch}{1.1}
\resizebox{\linewidth}{30mm}{
\setlength{\tabcolsep}{7mm}{
\begin{tabular}{c!{\vrule width 0.1pt}cc!{\vrule width 0.1pt}cc!{\vrule width 0.1pt}cc}
\hline\thickhline
\rowcolor{gray!80}
\multicolumn{1}{c!{\vrule width 0.1pt}}{\textbf{\textcolor{white}{Tasks}}} 
& \multicolumn{2}{c!{\vrule width 0.1pt}}{\textbf{\textcolor{white}{Node Classification}}} 
& \multicolumn{2}{c!{\vrule width 0.1pt}}{\textbf{\textcolor{white}{Link Prediction}}} 
& \multicolumn{2}{c}{\textbf{\textcolor{white}{Node Clustering}}} \\
\hline
\rowcolor{gray!10}
& \multicolumn{2}{c!{\vrule width 0.1pt}}{\textbf{Movies}} 
& \multicolumn{2}{c!{\vrule width 0.1pt}}{\textbf{DY}} 
& \multicolumn{2}{c}{\textbf{Toys}} \\
\cline{2-7}
\rowcolor{gray!10}
\multirow{-2}{*}{\diagbox[width=12em,height=2.4em]{\textbf{Methods}}{\textbf{Datasets}}} 
& ACC & F1-Score 
& MRR & Hits@3 
& NMI & ARI \\
\hline

GCN     
& $52.24_{\pm 0.15}$ & $41.80_{\pm 0.13}$ 
& $70.22_{\pm 0.35}$ & $83.80_{\pm 0.37}$ 
& $46.50_{\pm 0.15}$ & $31.66_{\pm 1.21}$ \\

\rowcolor{gray!10}
GCNII   
& $52.05_{\pm 0.28}$ & $42.62_{\pm 0.30}$ 
& $71.17_{\pm 0.20}$ & $85.69_{\pm 0.33}$ 
& $48.72_{\pm 0.20}$ & $32.84_{\pm 1.33}$ \\

GAT 
& $51.92_{\pm 0.32}$ & $41.47_{\pm 0.29}$ 
& $70.35_{\pm 0.24}$ & $81.92_{\pm 0.52}$ 
& $47.34_{\pm 0.22}$ & $31.13_{\pm 1.22}$ \\

\rowcolor{gray!10}
GATv2 
& $50.31_{\pm 0.43}$ & $39.62_{\pm 0.34}$ 
& $70.18_{\pm 0.36}$ & $83.72_{\pm 0.39}$ 
& $48.45_{\pm 0.18}$ & $32.72_{\pm 1.29}$ \\

\hline

MMGCN 
& $56.32_{\pm 0.18}$ & $44.52_{\pm 0.20}$ 
& $72.53_{\pm 0.42}$ & $85.78_{\pm 0.51}$ 
& $49.37_{\pm 0.28}$ & $33.42_{\pm 1.41}$ \\

\rowcolor{gray!10}
MGAT 
& $55.51_{\pm 0.33}$ & $43.95_{\pm 0.47}$ 
& $72.19_{\pm 0.38}$ & $85.47_{\pm 0.52}$ 
& $48.78_{\pm 0.28}$ & $32.89_{\pm 1.35}$ \\

LGMRec 
& $55.44_{\pm 0.27}$ & $42.78_{\pm 0.24}$ 
& $73.55_{\pm 0.18}$ & $86.28_{\pm 0.34}$ 
& $48.82_{\pm 0.29}$ & $32.63_{\pm 2.28}$ \\

\rowcolor{gray!10}
MLaGA 
& $56.25_{\pm 0.40}$ & \textcolor{orange}{$\mathbf{44.67_{\pm 0.51}}$} 
& $72.11_{\pm 0.42}$ & $85.23_{\pm 0.29}$ 
& $49.20_{\pm 0.32}$ & $33.26_{\pm 1.48}$ \\

GraphGPT-O 
& $52.48_{\pm 0.21}$ & $38.35_{\pm 0.42}$ 
& $70.04_{\pm 0.62}$ & $81.22_{\pm 1.36}$ 
& $45.34_{\pm 0.31}$ & $30.85_{\pm 1.72}$ \\

\rowcolor{gray!10}
Graph4MM 
& $55.48_{\pm 0.27}$ & $42.95_{\pm 0.26}$ 
& $74.31_{\pm 0.30}$ & $86.59_{\pm 0.19}$ 
& $46.74_{\pm 0.32}$ & $32.17_{\pm 1.21}$ \\

InstructG2I 
& $54.37_{\pm 0.19}$ & $42.55_{\pm 0.34}$ 
& $72.38_{\pm 0.24}$ & $86.11_{\pm 0.26}$ 
& $47.41_{\pm 0.59}$ & $31.26_{\pm 2.14}$ \\

\rowcolor{gray!10}
DMGC 
& $52.26_{\pm 0.64}$ & $42.54_{\pm 0.71}$ 
& \textcolor{royalblue}{$\mathbf{75.22_{\pm 0.38}}$} 
& \textcolor{royalblue}{$\mathbf{91.14_{\pm 0.40}}$} 

& $48.50_{\pm 0.30}$ & $31.95_{\pm 1.30}$ \\

DGF 
& \textcolor{royalblue}{$\mathbf{56.48_{\pm 0.30}}$} 
& \textcolor{royalblue}{$\mathbf{46.14_{\pm 0.27}}$} 
& \textcolor{orange}{$\mathbf{74.95_{\pm 0.35}}$} 
& \textcolor{orange}{$\mathbf{87.60_{\pm 0.52}}$} 
& \textcolor{royalblue}{$\mathbf{50.50_{\pm 0.28}}$}
& \textcolor{royalblue}{$\mathbf{34.62_{\pm 1.32}}$} \\

\rowcolor{gray!10}
MIG-GT 
& $54.93_{\pm 0.47}$ & $43.12_{\pm 0.58}$ 
& $72.63_{\pm 0.50}$ & $84.74_{\pm 0.57}$ 
& $47.86_{\pm 0.38}$ & $31.94_{\pm 1.52}$ \\

NTSFormer 
& \textcolor{orange}{$\mathbf{56.37_{\pm 0.25}}$} 
& $43.83_{\pm 0.18}$ 
& $72.15_{\pm 0.36}$ & $85.50_{\pm 0.42}$ 
& \textcolor{orange}{$\mathbf{49.42_{\pm 0.38}}$} 
& \textcolor{orange}{$\mathbf{34.01_{\pm 1.35}}$} \\

\rowcolor{gray!10}
UniGraph2 
& $54.51_{\pm 0.59}$ & $42.53_{\pm 0.42}$ 
& $71.23_{\pm 0.46}$ & $83.31_{\pm 0.29}$ 
& $47.17_{\pm 0.29}$ & $31.81_{\pm 1.30}$ \\

\hline

TMTE (Ours) 
& \textcolor{darkred}{$\mathbf{60.31_{\pm 0.24}}$} 
& \textcolor{darkred}{$\mathbf{53.48_{\pm 0.18}}$} 
& \textcolor{darkred}{$\mathbf{78.61_{\pm 0.24}}$} 
& \textcolor{darkred}{$\mathbf{93.20_{\pm 0.42}}$} 
& \textcolor{darkred}{$\mathbf{54.66_{\pm 0.30}}$} 
& \textcolor{darkred}{$\mathbf{39.52_{\pm 1.01}}$} \\

\hline\thickhline
\end{tabular}
}}
\vspace{0.15cm}
\end{table*}

\section{Experiments}
\label{sec: experiments}

In this section, we first describe the experimental setup, with full reproducibility details deferred to Appendix~\ref{appendix: dataset details} and Appendix~\ref{appendix: more experimental setups}. 
We then conduct comprehensive empirical evaluations to answer the following research questions: \textbf{Q1}: Does TMTE outperform existing unimodal GNNs and MGL methods on MAG datasets? 
\textbf{Q2}: What are the individual contributions of each module in TMTE? 
\textbf{Q3}: How robust is TMTE under noisy settings? 
\textbf{Q4}: Does TMTE outperform existing unimodal GNNs and MGL methods on non-graph multimodal datasets? 
\textbf{Q5}: What are the time and memory overheads, and how efficient is convergence? 
\textbf{Q6}: How robust is TMTE to variations in hyperparameters? Moreover, we provide additional experimental results on additional MAG datasets in Appendix~\ref{appendix: more experiments}.

\subsection{Experimental Setup}
\label{sec: exp setup}

\noindent \textbf{Datasets.}
We evaluate TMTE on 9 publicly MAG datasets and 1 non-graph multimodal dataset, covering diverse domains including social networks, recommendation systems, art, vision-language, and literature.
Specifically, the graph datasets include RedditS~\cite{RedditS} (social network), Movies~\cite{ni2019justifying} (movie network), two recommendation networks (Grocery, Ele-fashion)~\cite{ni2019justifying,hou2024bridging}, two video networks (DY, Bili Dance)~\cite{zhang2024ninerec}, SemArt~\cite{garcia2018_SemArt} (art network), Flickr30k~\cite{plummer2015_Flickr30k} (image-text network), and Goodreads~\cite{goodreads_1, goodreads_2} (book network). The non-graph dataset is MVSA~\cite{MVSA}.
Due to space limitations, detailed statistics and dataset descriptions are provided in Appendix~\ref{appendix: dataset details}.

\vspace{+0.1cm}
\noindent \textbf{Baselines.} We consider two categories of methods: (1) \textit{Unimodal GNNs}: GCN~\cite{gcn}, GCNII~\cite{gcnii}, GAT~\cite{gat}, and GATv2~\cite{gatv2}. (2) \textit{MGL methods}: MMGCN~\cite{mmgcn}, MGAT~\cite{mgat}, LGMRec~\cite{lgmrec}, MLaGA~\cite{mlaga}, GraphGPT-O~\cite{graphgpt_o}, Graph4MM~\cite{graph4mm}, InstructG2I~\cite{instructg2i}, DMGC~\cite{dmgc}, DGF~\cite{dgf}, MIG-GT~\cite{mig_gt}, NTSFormer~\cite{ntsformer}, and UniGraph2~\cite{unigraph2}. The details of these baselines are provided in Appendix~\ref{appendix: baseline details}.

\vspace{+0.1cm}
\noindent \textbf{Downstream Tasks.} We conduct comprehensive evaluations of these methods on: (1) \textit{Graph-centric tasks}: node classification, link prediction, and node clustering, and (2) \textit{Modality-centric tasks}: cross-modal retrieval, G2Text, and G2Image. Given the complexity of evaluation pipelines, hyperparameter configurations, and metrics, we provide full details in Appendix~\ref{appendix: more experimental setups}.

\subsection{Main Results (Answer for Q1)}
\label{sec: main results}

To answer \textbf{Q1}, we report results on three graph-centric tasks in Table~\ref{tab: graph_centric} and three modality-centric tasks in Table~\ref{tab: modality_centric}. We further evaluate these methods on a non-graph dataset in Appendix~\ref{appendix: non_graph_dataset}.

\vspace{+0.1cm}
\noindent \textbf{Graph Tasks}.
 TMTE consistently achieves the best performance across all datasets and evaluation metrics, demonstrating stable superiority over existing methods. Specifically, for node classification, TMTE improves over the second-best method by up to +3.83\% in Accuracy and +7.34\% in F1-score. For link prediction, TMTE achieves improvements of +3.39\% in MRR and +2.06\% in Hits@3 over the runner-up; for node clustering, TMTE improves NMI and ARI by +4.16\% and +5.51\% compared with the second-best results. Notably, although MGL methods generally surpass unimodal GNNs, none maintains consistent dominance across all tasks and metrics.

\vspace{+0.1cm}
\noindent \textbf{Modality Tasks}.
Our modality-centric tasks include retrieval and generation, where TMTE consistently achieves state-of-the-art results. On Ele-fashion retrieval, TMTE reaches MRR 95.22 and Hits@3 86.48, surpassing UniGraph2 by 2.70\% and 1.23\%. On Flickr30k G2Text, it outperforms Graph4MM by 1.05\% BLEU-4 and 2.68\% CIDEr. For SemArt G2Image, TMTE exceeds InstructG2I by 3.63\% CLIP-S and 1.87\% DINOv2-S, and GraphGPT-O by 3.74\% CLIP-S and 3.61\% DINOv2-S. These results highlight that fixed, task-agnostic graph topologies often fail in MGL, as different generation objectives require distinct semantic cues, making suboptimal structures inevitable in existing methods.
    
TMTE performs task-aware modality-topology co-evolution, where the procedure induces an adaptive graph structure tailored to the downstream objective, and the refined topology in turn enhances contextual multimodal representations. The consistent superiority of TMTE across all graph-centric and modality-centric tasks further confirms that dynamically evolving topology toward task-specific objectives is essential for effective MGL.

\begin{table*}[htbp]
\setlength{\abovecaptionskip}{0.2cm}
\setlength{\belowcaptionskip}{-0.2cm}
\centering
\caption{Performance comparison on three \textbf{modality-centric} downstream tasks. The best, second best and third best results are highlighted in \textcolor{darkred}{\textbf{red}}, \textcolor{royalblue}{\textbf{blue}} and \textcolor{orange}{\textbf{orange}}, respectively.}
\label{tab: modality_centric}
\footnotesize 
\renewcommand{\arraystretch}{1.1}
\resizebox{\linewidth}{30mm}{
\setlength{\tabcolsep}{7mm}{
\begin{tabular}{c!{\vrule width 0.1pt}cc!{\vrule width 0.1pt}cc!{\vrule width 0.1pt}cc}
\hline\thickhline
\rowcolor{gray!80}
\multicolumn{1}{c!{\vrule width 0.1pt}}{\textbf{\textcolor{white}{Tasks}}} 
& \multicolumn{2}{c!{\vrule width 0.1pt}}{\textbf{\textcolor{white}{Modality Retrieval}}} 
& \multicolumn{2}{c!{\vrule width 0.1pt}}{\textbf{\textcolor{white}{G2Text}}} 
& \multicolumn{2}{c}{\textbf{\textcolor{white}{G2Image}}} \\
\hline
\rowcolor{gray!10}
& \multicolumn{2}{c!{\vrule width 0.1pt}}{\textbf{Ele-fashion}} 
& \multicolumn{2}{c!{\vrule width 0.1pt}}{\textbf{Flickr30k}} 
& \multicolumn{2}{c}{\textbf{SemArt}} \\
\cline{2-7}
\rowcolor{gray!10}
\multirow{-2}{*}{\diagbox[width=12em,height=2.4em]{\textbf{Methods}}{\textbf{Datasets}}} 
& MRR & Hits@3
& BLEU-4 & CIDEr 
& CLIP-S & DINOv2-S \\
\hline

GCN
& $76.40_{\pm 0.44}$
& $68.55_{\pm 0.52}$
& $5.62_{\pm 0.23}$
& $38.43_{\pm 1.27}$
& $50.05_{\pm 0.18}$
& $35.52_{\pm 0.21}$ \\

\rowcolor{gray!10}
GCNII
& $77.41_{\pm 0.36}$
& $70.56_{\pm 0.48}$
& $5.27_{\pm 0.19}$
& $39.52_{\pm 1.11}$
& $49.51_{\pm 0.24}$
& $35.22_{\pm 0.17}$ \\

GAT
& $78.12_{\pm 0.57}$
& $70.35_{\pm 0.39}$
& $5.12_{\pm 0.21}$
& $39.48_{\pm 1.09}$
& $51.26_{\pm 0.33}$
& $36.24_{\pm 0.42}$ \\

\rowcolor{gray!10}
GATv2
& $78.24_{\pm 0.33}$
& $70.79_{\pm 0.58}$
& $5.28_{\pm 0.25}$
& $39.81_{\pm 1.36}$
& $51.48_{\pm 0.46}$
& $36.41_{\pm 0.90}$ \\

\hline

MMGCN
& $81.41_{\pm 0.52}$
& $73.56_{\pm 0.43}$
& $5.78_{\pm 0.17}$
& $43.52_{\pm 1.08}$
& $54.51_{\pm 0.37}$
& $39.46_{\pm 0.55}$ \\

\rowcolor{gray!10}
MGAT
& $81.70_{\pm 0.38}$
& $74.39_{\pm 0.49}$
& $6.49_{\pm 0.22}$
& $46.34_{\pm 1.47}$
& $55.38_{\pm 0.41}$
& $40.24_{\pm 0.63}$ \\

LGMRec
& $89.45_{\pm 0.41}$
& $80.26_{\pm 0.55}$
& $5.90_{\pm 0.18}$
& $60.40_{\pm 1.25}$
& $60.28_{\pm 0.52}$
& $44.87_{\pm 0.66}$ \\

\rowcolor{gray!10}
MLaGA
& $87.45_{\pm 0.59}$
& $79.22_{\pm 0.46}$
& $9.26_{\pm 0.31}$
& $70.94_{\pm 1.18}$
& $68.23_{\pm 0.73}$
& \textcolor{orange}{$\mathbf{52.88}_{\pm 0.84}$} \\

GraphGPT-O
& $87.84_{\pm 0.47}$
& $80.12_{\pm 0.42}$
& \textcolor{orange}{$\mathbf{9.57_{\pm 0.28}}$}
& \textcolor{orange}{$\mathbf{72.26_{\pm 1.09}}$}
& \textcolor{orange}{$\mathbf{70.47_{\pm 0.69}}$}
& $53.68_{\pm 0.77}$ \\

\rowcolor{gray!10}
Graph4MM
& $85.94_{\pm 0.35}$
& $78.13_{\pm 0.53}$
& \textcolor{royalblue}{$\mathbf{10.15_{\pm 0.33}}$}
& \textcolor{royalblue}{$\mathbf{74.46_{\pm 1.26}}$}
& $66.85_{\pm 0.58}$
& $51.28_{\pm 0.69}$ \\

InstructG2I
& $88.47_{\pm 0.48}$
& $80.76_{\pm 0.57}$
& $9.42_{\pm 0.29}$
& $71.25_{\pm 1.31}$
& \textcolor{royalblue}{$\mathbf{70.58_{\pm 0.75}}$}
& \textcolor{royalblue}{$\mathbf{55.42_{\pm 0.88}}$} \\

\rowcolor{gray!10}
DMGC
& $91.34_{\pm 0.39}$
& \textcolor{orange}{$\mathbf{85.22_{\pm 0.44}}$}
& $7.41_{\pm 0.27}$
& $60.95_{\pm 1.19}$
& $60.87_{\pm 0.56}$
& $46.35_{\pm 0.71}$ \\

DGF
& $90.84_{\pm 0.42}$
& $82.59_{\pm 0.36}$
& $6.75_{\pm 0.47}$
& $62.17_{\pm 1.20}$
& $61.53_{\pm 1.20}$
& $46.02_{\pm 1.14}$ \\

\rowcolor{gray!10}
MIG-GT
& $91.33_{\pm 0.51}$
& $83.27_{\pm 0.47}$
& $8.30_{\pm 0.22}$
& $64.27_{\pm 1.33}$
& $63.45_{\pm 0.67}$
& $47.62_{\pm 0.78}$ \\

NTSFormer
& \textcolor{orange}{$\mathbf{91.42_{\pm 0.37}}$}
& $83.50_{\pm 0.41}$
& $8.41_{\pm 0.24}$
& $65.29_{\pm 1.28}$
& $62.88_{\pm 0.63}$
& $47.81_{\pm 0.74}$ \\

\rowcolor{gray!10}
UniGraph2
& \textcolor{royalblue}{$\mathbf{92.52_{\pm 0.35}}$}
& \textcolor{royalblue}{$\mathbf{85.25_{\pm 0.41}}$}
& $8.52_{\pm 0.26}$
& $64.38_{\pm 1.42}$
& $63.48_{\pm 0.58}$
& $47.53_{\pm 0.61}$ \\

\hline
TMTE (Ours)
& \textcolor{darkred}{$\mathbf{95.22_{\pm 0.33}}$}
& \textcolor{darkred}{$\mathbf{86.48_{\pm 0.37}}$}
& \textcolor{darkred}{$\mathbf{11.20_{\pm 0.15}}$}
& \textcolor{darkred}{$\mathbf{77.14_{\pm 1.33}}$}
& \textcolor{darkred}{$\mathbf{74.21_{\pm 0.44}}$}
& \textcolor{darkred}{$\mathbf{57.29_{\pm 0.52}}$} \\

\hline\thickhline
\end{tabular}
}}
\vspace{0.15cm}
\end{table*}

\subsection{Ablation Study (Answer for Q2)}
\label{sec: ablation study}

To address \textbf{Q2}, we conduct an ablation study to assess the contribution of each TMTE module. As described in Sec.~\ref{sec: introduction} and Sec.~\ref{sec: methodology}, TMTE comprises three core mechanisms: topology evolution, modality evolution, and task-aware optimization.

Due to the complex interactions among these mechanisms, individual modules cannot be simply removed. Therefore, we define three variants to evaluate their effects: (1) \textit{One-shot Topology Evolution (One-shot TE)}: The MAG topology evolves only once per epoch based on the original modality features, without co-evolving with latent modality representations; (2) \textit{Only Modality Evolution (Only ME)}: The topology evolution is removed, and modality evolution relies solely on the original topology rather than the evolved one; and (3) \textit{Task-agnostic Evolution (Task-agnostic E)}: Modality and topology co-evolution ignores downstream task objectives, performing task-agnostic evolution followed by fine-tuning on the tasks.

\begin{table*}[htbp]
\setlength{\abovecaptionskip}{0.2cm}
\setlength{\belowcaptionskip}{-0.2cm}
\centering
\caption{\textbf{Ablation study} on 6 MAG datasets with graph-centric and modality-centric tasks.}
\label{tab: ablation study}
\footnotesize 
\renewcommand{\arraystretch}{1.1}
\resizebox{\linewidth}{22mm}{
\setlength{\tabcolsep}{7mm}{
\begin{tabular}{c!{\vrule width 0.1pt}cc!{\vrule width 0.1pt}cc!{\vrule width 0.1pt}cc}
\hline\thickhline
\rowcolor{gray!80}
\multicolumn{1}{c!{\vrule width 0.1pt}}{\textbf{\textcolor{white}{Tasks}}} 
& \multicolumn{2}{c!{\vrule width 0.1pt}}{\textbf{\textcolor{white}{Node Classification}}} 
& \multicolumn{2}{c!{\vrule width 0.1pt}}{\textbf{\textcolor{white}{Link Prediction}}} 
& \multicolumn{2}{c}{\textbf{\textcolor{white}{Node Clustering}}} \\
\hline
\rowcolor{gray!10}
& \multicolumn{2}{c!{\vrule width 0.1pt}}{\textbf{Movies}} 
& \multicolumn{2}{c!{\vrule width 0.1pt}}{\textbf{DY}} 
& \multicolumn{2}{c}{\textbf{Toys}} \\
\cline{2-7}
\rowcolor{gray!10}
\multirow{-2}{*}{\diagbox[width=12em,height=2.4em]{\textbf{Methods}}{\textbf{Datasets}}} 
& ACC & F1-Score 
& MRR & Hits@3 
& NMI & ARI \\
\hline

\textit{One-shot TE}
& \textcolor{orange}{$\mathbf{57.44_{\pm 0.24}}$} 
& \textcolor{orange}{$\mathbf{47.10_{\pm 0.19}}$} 
& \textcolor{royalblue}{$\mathbf{77.15_{\pm 0.40}}$} 
& \textcolor{royalblue}{$\mathbf{92.47_{\pm 0.29}}$} 
& $48.33_{\pm 0.24}$
& \textcolor{orange}{$\mathbf{33.27_{\pm 1.25}}$} \\

\rowcolor{gray!10}
\textit{Only ME}
& $56.31_{\pm 0.29}$
& $44.27_{\pm 0.30}$ 
& $73.45_{\pm 0.19}$ 
& $88.57_{\pm 0.42}$ 
& \textcolor{orange}{$\mathbf{48.40_{\pm 0.14}}$} 
& $32.24_{\pm 0.82}$ \\

\textit{Task-agnostic E}
& \textcolor{royalblue}{$\mathbf{58.75_{\pm 0.32}}$} 
& \textcolor{royalblue}{$\mathbf{49.61_{\pm 0.24}}$} 
& \textcolor{orange}{$\mathbf{76.18_{\pm 0.13}}$} 
& \textcolor{orange}{$\mathbf{91.52_{\pm 0.38}}$} 
& \textcolor{royalblue}{$\mathbf{52.30_{\pm 0.25}}$} 
& \textcolor{royalblue}{$\mathbf{36.49_{\pm 1.44}}$} \\


\rowcolor{gray!10}
TMTE (Ours) 
& \textcolor{darkred}{$\mathbf{60.31_{\pm 0.24}}$} 
& \textcolor{darkred}{$\mathbf{53.48_{\pm 0.18}}$} 
& \textcolor{darkred}{$\mathbf{78.61_{\pm 0.24}}$} 
& \textcolor{darkred}{$\mathbf{93.20_{\pm 0.42}}$} 
& \textcolor{darkred}{$\mathbf{54.66_{\pm 0.30}}$} 
& \textcolor{darkred}{$\mathbf{39.52_{\pm 1.01}}$} \\

\hline\thickhline
\rowcolor{gray!80}
\multicolumn{1}{c!{\vrule width 0.1pt}}{\textbf{\textcolor{white}{Tasks}}} 
& \multicolumn{2}{c!{\vrule width 0.1pt}}{\textbf{\textcolor{white}{Modality Retrieval}}} 
& \multicolumn{2}{c!{\vrule width 0.1pt}}{\textbf{\textcolor{white}{G2Text}}} 
& \multicolumn{2}{c}{\textbf{\textcolor{white}{G2Image}}} \\
\hline
\rowcolor{gray!10}
& \multicolumn{2}{c!{\vrule width 0.1pt}}{\textbf{Ele-fashion}} 
& \multicolumn{2}{c!{\vrule width 0.1pt}}{\textbf{Flickr30k}} 
& \multicolumn{2}{c}{\textbf{SemArt}} \\
\cline{2-7}
\rowcolor{gray!10}
\multirow{-2}{*}{\diagbox[width=12em,height=2.4em]{\textbf{Methods}}{\textbf{Datasets}}} 
& MRR & Hits@3
& BLEU-4 & CIDEr 
& CLIP-S & DINOv2-S \\
\hline

\textit{One-shot TE}
& \textcolor{orange}{$\mathbf{91.27_{\pm 0.42}}$} 
& \textcolor{orange}{$\mathbf{84.09_{\pm 0.23}}$} 
& \textcolor{orange}{$\mathbf{9.61_{\pm 0.30}}$} 
& \textcolor{orange}{$\mathbf{72.48_{\pm 0.17}}$} 
& \textcolor{royalblue}{$\mathbf{73.22_{\pm 0.48}}$} 
& \textcolor{orange}{$\mathbf{54.30_{\pm 0.38}}$} \\

\rowcolor{gray!10}
\textit{Only ME}
& $87.95_{\pm 0.42}$ 
& $82.26_{\pm 0.29}$
& $8.32_{\pm 0.35}$
& $67.42_{\pm 0.25}$ 
& $66.42_{\pm 0.35}$ 
& $52.24_{\pm 0.84}$ \\

\textit{Task-agnostic E}
& \textcolor{royalblue}{$\mathbf{94.10_{\pm 0.27}}$}
& \textcolor{royalblue}{$\mathbf{85.75_{\pm 0.30}}$}
& \textcolor{royalblue}{$\mathbf{10.82_{\pm 0.33}}$}
& \textcolor{royalblue}{$\mathbf{76.42_{\pm 1.58}}$}
& \textcolor{orange}{$\mathbf{72.37_{\pm 0.22}}$}
& \textcolor{royalblue}{$\mathbf{55.81_{\pm 0.61}}$} \\

\rowcolor{gray!10}
TMTE (Ours)
& \textcolor{darkred}{$\mathbf{95.22_{\pm 0.33}}$}
& \textcolor{darkred}{$\mathbf{86.48_{\pm 0.37}}$}
& \textcolor{darkred}{$\mathbf{11.20_{\pm 0.15}}$}
& \textcolor{darkred}{$\mathbf{77.14_{\pm 1.33}}$}
& \textcolor{darkred}{$\mathbf{74.21_{\pm 0.44}}$}
& \textcolor{darkred}{$\mathbf{57.29_{\pm 0.52}}$} \\

\hline\thickhline
\end{tabular}
}}
\vspace{0.15cm}
\end{table*}

\begin{figure*}[htb]
 \centering
  \includegraphics[width=0.998\textwidth]{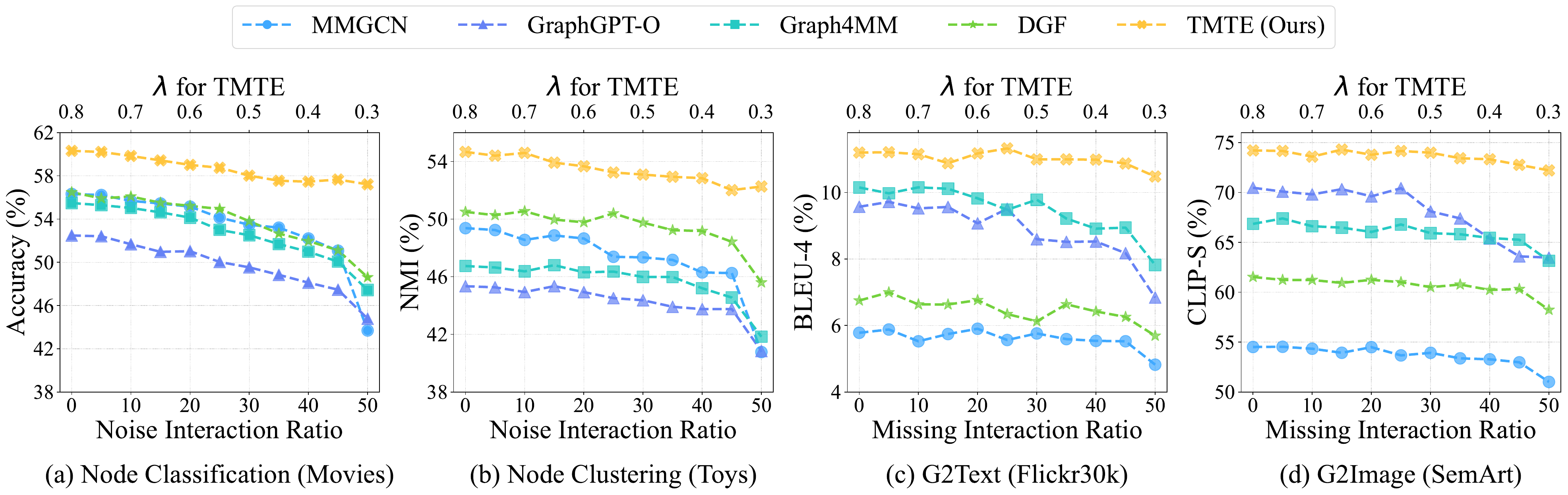}
  \caption{\textbf{Experimental results of our robustness analysis}. We investigate two types of topological noise, including \textbf{noise interactions} (i.e., randomly adding edges), which are presented in (a) and (b); and \textbf{missing interactions} (i.e., randomly removing edges), which are presented in (c) and (d).}
\label{fig: robust}
\end{figure*}

Based on this, Table~\ref{tab: ablation study} confirms the importance of these core mechanisms. Each variant induces a noticeable performance drop. Notably, \textit{Only Modality Evolution} generally causes the largest degradation, highlighting the critical role of topology evolution. The performance drop from \textit{One-shot Topology Evolution} is often more pronounced than that of \textit{Task-agnostic Evolution}, indicating that multi-round co-evolution of modality and topology progressively guides the learning of better structures, which single-round evolution based on the original modality space cannot achieve. Finally, the decrease under \textit{Task-agnostic Evolution} demonstrates the necessity of incorporating downstream objectives into the co-evolution process. Overall, all components are essential for TMTE's performance and contribute significantly to its overall effectiveness.

\subsection{Robustness Analysis (Answer for Q3)}
\label{sec: robustness study}
To answer \textbf{Q3}, we investigate the robustness of TMTE under topology noise scenarios, including noise interaction and missing interaction, which stem from the inherent topological limitations of real-world MAGs discussed in Sec.~\ref{sec: introduction}. As illustrated in Fig.~\ref{fig: robust}, when the level of topological noise gradually increases, TMTE maintains remarkable robustness and stable performance by appropriately adjusting the $\lambda$ parameter (Eqs.~\eqref{eq: combine_adj_norm_t} and~\eqref{eq: combine_adj_norm_t_repeat}), which effectively controls the weight of the original graph topology during the topological evolution process and ensures a balanced integration of evolving structural information. In contrast, the baseline method exhibits a significant  performance decline as the noise level increases.

\subsection{Performances on Non-graph Datasets (Answer for Q4)}
\label{sec: non_graph_dataset}

\begin{figure}
    \centering
    \includegraphics[width=0.25\textwidth]{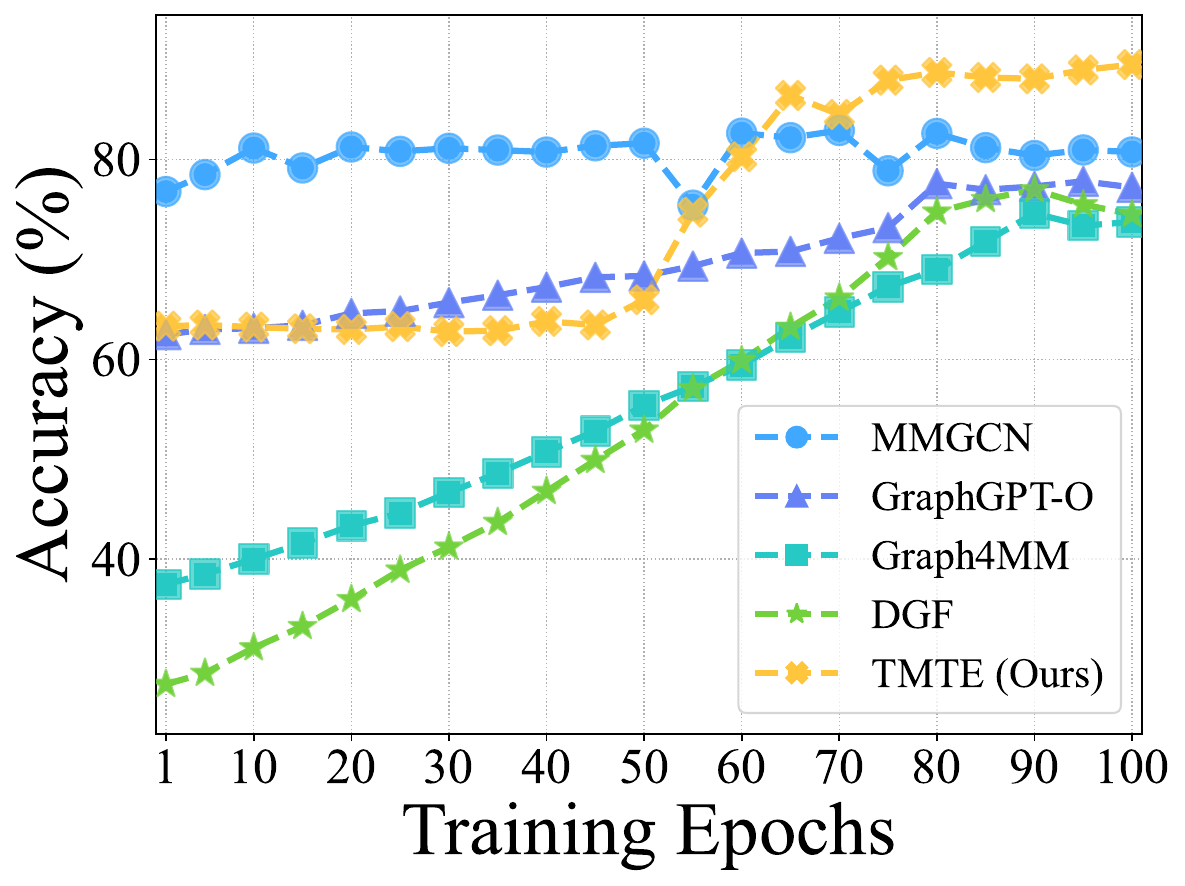}
    \caption{\textbf{Accuracy curves} under the MVSA dataset.}
    \label{fig: MVSA}
\end{figure}
    
To further evaluate TMTE on multimodal data without explicitly constructed topological structures, we conduct experiments on the non-graph MVSA dataset~\cite{MVSA}, with details provided in Appendix~\ref{appendix: dataset details}. The results are shown in Fig.~\ref{fig: MVSA}. As observed, TMTE maintains an accuracy of approximately 63\% during the first 50 epochs, followed by rapid improvement that approaches convergence around the 75th epoch and achieves the best overall performance. In contrast, other methods exhibit less favorable optimization behavior, either converging more slowly (e.g., DGF) or quickly plateauing at suboptimal levels (e.g., MMGCN), demonstrating the superior convergence and robustness of TMTE without predefined topological structures.

\subsection{Efficiency Analysis (Answer for Q5)}
\label{sec: efficiency analysis}

\begin{table}[htb]
\setlength{\abovecaptionskip}{0.2cm}
\setlength{\belowcaptionskip}{-0.2cm}
\centering
\footnotesize 
\renewcommand{\arraystretch}{1.1}
\setlength{\tabcolsep}{5mm}

\caption{\textbf{Per-epoch efficiency} on Movies.}
\label{tab: efficiency}

\resizebox{1\columnwidth}{!}{
\begin{tabular}{l!{\vrule width 0.1pt}c!{\vrule width 0.1pt}c!{\vrule width 0.1pt}c}
\hline\thickhline
\rowcolor{gray!80}
\textbf{\textcolor{white}{Method}} &
\textbf{\textcolor{white}{E-Train. (s)}} &
\textbf{\textcolor{white}{E-Infer. (s)}} &
\textbf{\textcolor{white}{Param.}} \\
\hline
MMGCN & \textcolor{royalblue}{$\mathbf{0.0891}$} & \textcolor{royalblue}{$\mathbf{0.0352}$} & $15.8$M \\
\rowcolor{gray!10}
GraphGPT-O & $1.8257$ & $0.5167$ & $48.7$M \\

Graph4MM & $0.1599$ & $0.5311$ & $33.2$M \\
\rowcolor{gray!10}
InstructG2I & $24.9133$ & $0.8211$ & $62.3$M \\
DMGC & \textcolor{orange}{$\mathbf{0.1362}$} & $0.0441$ & \textcolor{orange}{$\mathbf{0.85}$M} \\
\rowcolor{gray!10}
DGF & $0.1599$ & \textcolor{orange}{$\mathbf{0.0391}$} & $0.9$M \\
NTSFormer & \textcolor{darkred}{$\mathbf{0.0611}$} & \textcolor{darkred}{$\mathbf{0.0206}$} & \textcolor{royalblue}{$\mathbf{0.7}$M} \\
\rowcolor{gray!10}
Unigraph2 & $0.4933$ & $0.0646$ & $36.6$M \\
TMTE (Ours) & $0.4225$ & $0.1099$ & \textcolor{darkred}{$\mathbf{0.5}$M} \\
\hline\thickhline
\end{tabular}
}
\end{table}

To answer \textbf{Q4}, we provide a detailed evaluation of TMTE in terms of computation and storage efficiency. As shown in Table~\ref{tab: efficiency}, TMTE's per-epoch training time ($0.4225$ s) and inference time ($0.1099$ s) are higher than the fastest baselines such as NTSFormer ($0.0611$ s) and MMGCN ($0.0891$ s), reflecting a moderate additional computational cost per batch. Notably, despite this overhead, TMTE maintains a very small parameter footprint ($0.5$M). This modest increase in training and inference time is acceptable considering TMTE’s consistent performance advantages in both graph-centric and modality-centric tasks. In summary, TMTE achieves superior effectiveness while remaining deployable, providing a favorable trade-off between efficiency and task performance. Moreover, we provide additional theoretical analysis of TMTE on large-scale graphs in terms of stability and convergence guarantee (Theorems~\ref{thm: inexact stability}and~\ref{thm: contraction large scale}).

\subsection{Hyperparameter Analysis (Answer for Q6)}
\label{sec: hyperparameter}

\begin{figure*}
 \centering
  \includegraphics[width=0.998\textwidth]{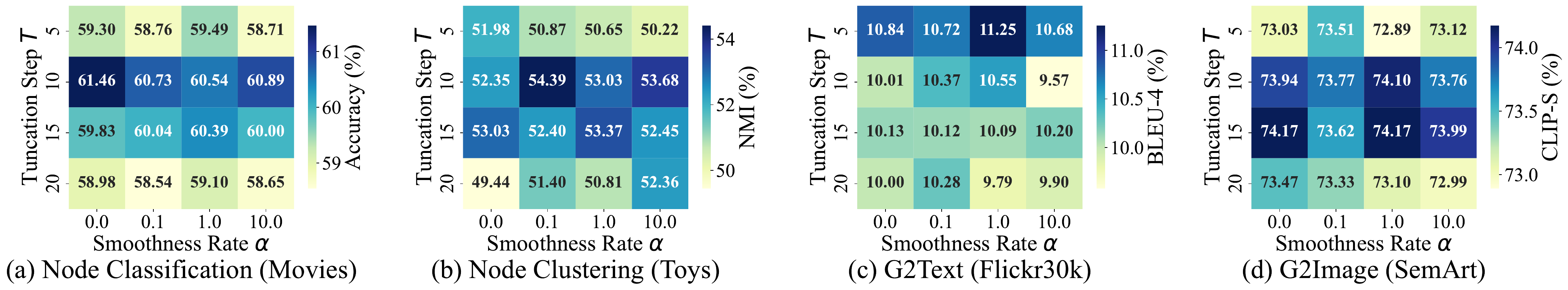}
  \caption{\textbf{Hyperparameter analysis for $\alpha$ and $T$} on four datasets and tasks.}
\label{fig: heatmap_T_alpha}
\end{figure*}

In this section, we investigate the impacts of key hyperparameters in TMTE. Specifically, our evaluations focus on: (1) the number of perspectives for similarity metric learning (i.e., $K$ in Eq.~\eqref{eq: sim channel}); (2) the trade-off parameter between the original topology and the evolved topology (i.e., $\lambda$ in Eqs.~\eqref{eq: combine_adj_norm_t} and~\eqref{eq: combine_adj_norm_t_repeat}); (3) the smoothness rate for modality representation (i.e., $\alpha$ in Eq.~\eqref{eq: smooth}); and (4) the number of truncation steps for power series approximation (i.e., $T$ in Eq.~\eqref{eq: approximated}). For any TMTE hyperparameters not discussed, we describe their fixed values or search ranges in Appendix~\ref{appendix: more experimental setups}.

\vspace{+0.1cm}
\noindent \textbf{Hyperparameters for Topology Evolution ($K$ and $\lambda$).}
For $K$, which denotes the number of perspectives in similarity metric learning (Eq.~\eqref{eq: sim channel}), it controls the diversity of learned perspectives during topology evolution. A small $K$ may limit the model’s ability to capture heterogeneous structural patterns, while an excessively large $K$ can introduce redundancy and overfitting.

As shown in Table~\ref{tab: hyper_k}, performance first improves and then declines as $K$ increases. Specifically, increasing $K$ from 1 to 4 significantly boosts performance on Movies (ACC improves from 57.59\% to 60.31\%, and F1-Score from 48.23\% to 53.48\%), demonstrating that multi-perspective similarity learning effectively enhances the quality of learned topology. However, when $K$ further increases to 8 or 16, performance slightly drops. This suggests that too many perspectives may introduce noisy or redundant similarity channels, weakening discriminative ability. Therefore, a moderate number of perspectives (e.g., $K=4$) achieves the best trade-off between expressiveness and robustness.

\begin{table}
\setlength{\abovecaptionskip}{0.2cm}
\setlength{\belowcaptionskip}{-0.2cm}
\centering
\footnotesize 
\renewcommand{\arraystretch}{1.1}
\setlength{\tabcolsep}{1.5mm}

\caption{\textbf{Impact} of $K$ on Movies.}
\label{tab: hyper_k}

\resizebox{1\columnwidth}{10mm}{
\setlength{\tabcolsep}{7mm}{
\begin{tabular}{c|cc}
\hline\thickhline
\rowcolor{gray!80}

\cline{2-3}
\rowcolor{gray!10}
$K$ & ACC & F1-Score \\
\hline

$1$
& \textcolor{orange}{$\mathbf{57.59_{\pm 0.17}}$} 
& \textcolor{orange}{$\mathbf{48.23_{\pm 0.35}}$} \\

\rowcolor{gray!10}
$4$
& \textcolor{darkred}{$\mathbf{60.31_{\pm 0.24}}$} 
& \textcolor{darkred}{$\mathbf{53.48_{\pm 0.18}}$} \\

$8$
& \textcolor{royalblue}{$\mathbf{59.40_{\pm 0.32}}$} 
& \textcolor{royalblue}{$\mathbf{51.26_{\pm 0.53}}$} \\

\rowcolor{gray!10}
$16$
& $57.38_{\pm 0.29}$ 
& $47.96_{\pm 0.33}$ \\

\hline\thickhline
\end{tabular}
}
}
\end{table}

For $\lambda$, which balances the original topology and the evolved topology (Eqs.~\eqref{eq: combine_adj_norm_t} and~\eqref{eq: combine_adj_norm_t_repeat}), we have discussed its effect in Sec.~\ref{sec: robustness study}. As shown in Fig.~\ref{fig: robust}, $\lambda$ plays a crucial role in noisy topology scenarios. When topological noise increases, appropriately adjusting $\lambda$ enables TMTE to reduce reliance on corrupted original structures while leveraging the evolved topology. This adaptive weighting mechanism ensures stable performance under varying noise levels.

\vspace{+0.1cm}
\noindent \textbf{Hyperparameters for Modality Evolution ($\alpha$ and $T$).} As illustrated in Fig.~\ref{fig: heatmap_T_alpha}, we analyze the smoothness rate $\alpha$ (Eq.~\eqref{eq: smooth}) and the truncation step $T$ (Eq.~\ref{eq: approximated}) across four datasets and tasks, including Movies (node classification), Toys (node clustering), Flickr30k (G2Text), and SemArt (G2Image).

For $\alpha$, which controls the strength of modality representation smoothing, we observe that varying $\alpha$ across a wide range has relatively limited impact on the final performance. As shown in Fig.~\ref{fig: heatmap_T_alpha}, although extremely small or large values may lead to slight fluctuations, the overall performance remains consistently stable across different datasets. This indicates that TMTE is not sensitive to the choice of $\alpha$. In other words, the modality evolution mechanism maintains effectiveness under different smoothing strengths.

For $T$, it determines the approximation depth of the power series expansion. Small $T$ values (e.g., 5) may under-approximate higher-order propagation. Increasing $T$ to moderate levels (e.g., 10 or 15) generally improves performance, as higher-order interactions are better captured. However, further increasing $T$ (e.g., 20) slightly degrade performance and increase computational complexity.

\section{Related Works}

\textbf{Graph Neural Networks (GNNs).} Earlier research on deep graph learning extends convolution to handle graphs~\cite{bruna2013spectral} but comes with notable parameter counts. To this end, GCN~\cite{gcn} simplifies graph convolution by utilizing a first-order Chebyshev filter to capture local neighborhood information. Moreover, GAT~\cite{gat} adopts graph attention, allowing weighted aggregation. GraphSAGE~\cite{graphsage} introduces a variety of learnable aggregation functions for performing message aggregation. Moreover, recent studies extend GNN optimization from the centralized setting to the decentralized setting~\cite{fedtad, fedgta, fgl_unlearning, fedgfm, power, fedgala}. More GNN research can be found in surveys and benchmarks~\cite{wu2020comprehensive, zhou2020graph, data_centric_fgl_survey, openfgl}.

\vspace{+0.15cm} \noindent \textbf{Multimodal Graph Learning (MGL).} Multimodal graph learning (MGL)~\cite{ektefaie2023multimodal, wan2026openmag} aims to integrate heterogeneous modalities (e.g., vision, language, audio) within unified graph structures, enabling joint modeling of semantic and structural dependencies. Existing approaches generally extend classical graph learning techniques to multimodal settings, such as graph convolutions, attention mechanisms, hypergraph modeling, and contrastive objectives. In recommendation scenarios, prior studies~\cite{mmgcn, mgat, lgmrec, mig_gt} focus on constructing modality-aware interaction graphs and designing adaptive fusion strategies to disentangle collaborative and modality-specific signals. These methods improve robustness under sparse or noisy interactions by enhancing high-order propagation and long-range dependency modeling. Beyond recommendation, several works~\cite{dmgc, dgf, ntsformer, unigraph2} investigate more general multimodal attributed graph learning problems, including heterophily-aware modeling, feature denoising, cold-start node classification, and cross-domain representation learning. They typically emphasize structural filtering, contrastive alignment, and unified embedding spaces to obtain discriminative and transferable representations. With the rapid development of large language models (LLMs), recent efforts~\cite{yoon2023mmgl, mlaga, graphgpt_o, graph4mm, instructg2i, yan2025graph} explore graph-augmented multimodal reasoning and generation. These approaches align multimodal embeddings with topology and incorporate structural priors into foundation models, enabling structure-aware inference, controllable generation, and systematic benchmarking.

\vspace{0.05cm}
\section{Conclusion}
\label{sec: conclusion}

In this paper, we revisit the role of graph topology in multimodal-attributed graphs and identify its inherent limitations, including noisy connections, missing connections, and task-agnostic interactions. Motivated by the bidirectional coupling between modalities and topology, we propose a novel MGL framework, TMTE, enabling task-aware co-evolution of graph structure and multimodal representations. TMTE iteratively refines the initial topology using evolving modality embeddings, which in turn guide representation learning toward downstream objectives. Experiments show that TMTE effectively mitigates topology noise, consistently improving performance on both graph-centric and modality-centric tasks.



\newpage
\bibliographystyle{ACM-Reference-Format}
\balance
\bibliography{citation}

\appendix
\newpage

\section{Performances on Additional MAG Datasets}
\label{appendix: more experiments}

\begin{table*}[htbp]
\setlength{\abovecaptionskip}{0.2cm}
\setlength{\belowcaptionskip}{-0.2cm}
\centering
\caption{Additional performance comparison on Grocery, Bili Dance and RedditS datasets. The best, second best and third best results are highlighted in \textcolor{darkred}{\textbf{red}}, \textcolor{royalblue}{\textbf{blue}} and \textcolor{orange}{\textbf{orange}}, respectively.}
\label{tab: graph_centric_more}
\footnotesize 
\renewcommand{\arraystretch}{1.1}
\resizebox{\linewidth}{30mm}{
\setlength{\tabcolsep}{7mm}{
\begin{tabular}{c!{\vrule width 0.1pt}cc!{\vrule width 0.1pt}cc!{\vrule width 0.1pt}cc}
\hline\thickhline
\rowcolor{gray!80}
\multicolumn{1}{c!{\vrule width 0.1pt}}{\textbf{\textcolor{white}{Tasks}}} 
& \multicolumn{2}{c!{\vrule width 0.1pt}}{\textbf{\textcolor{white}{Node Classification}}} 
& \multicolumn{2}{c!{\vrule width 0.1pt}}{\textbf{\textcolor{white}{Link Prediction}}} 
& \multicolumn{2}{c}{\textbf{\textcolor{white}{Node Clustering}}} \\
\hline
\rowcolor{gray!10}
& \multicolumn{2}{c!{\vrule width 0.1pt}}{\textbf{Grocery}} 
& \multicolumn{2}{c!{\vrule width 0.1pt}}{\textbf{Bili Dance}} 
& \multicolumn{2}{c}{\textbf{RedditS}} \\
\cline{2-7}
\rowcolor{gray!10}
\multirow{-2}{*}{\diagbox[width=12em,height=2.4em]{\textbf{Methods}}{\textbf{Datasets}}} 
& ACC & F1-Score 
& MRR & Hits@3 
& NMI & ARI \\
\hline

GCN     
& $80.19_{\pm 0.30}$ & $72.25_{\pm 0.43}$ 
& $37.50_{\pm 0.22}$ & $47.30_{\pm 0.50}$ 
& $75.15_{\pm 0.25}$ & $71.08_{\pm 2.43}$ \\

\rowcolor{gray!10}
GCNII   
& $78.32_{\pm 0.24}$ & $70.12_{\pm 0.30}$ 
& $37.42_{\pm 0.26}$ & $47.41_{\pm 0.58}$ 
& $76.76_{\pm 0.30}$ & $72.29_{\pm 2.55}$ \\

GAT 
& $80.22_{\pm 0.37}$ & $72.38_{\pm 0.46}$ 
& $36.74_{\pm 0.31}$ & $46.62_{\pm 0.34}$ 
& $75.32_{\pm 0.27}$ & $70.98_{\pm 3.48}$ \\

\rowcolor{gray!10}
GATv2 
& $80.33_{\pm 0.42}$ & $73.05_{\pm 0.28}$ 
& $37.64_{\pm 0.28}$ & $47.41_{\pm 0.56}$ 
& $77.12_{\pm 0.28}$ & $72.33_{\pm 2.28}$ \\

\hline

MMGCN 
& $82.12_{\pm 0.38}$ & \textcolor{orange}{$\mathbf{74.91_{\pm 0.22}}$} 
& $38.29_{\pm 0.27}$ & $48.94_{\pm 0.46}$ 
& $77.59_{\pm 0.32}$ & $72.67_{\pm 2.53}$ \\

\rowcolor{gray!10}
MGAT 
& \textcolor{royalblue}{$\mathbf{82.48_{\pm 0.42}}$} 
& \textcolor{royalblue}{$\mathbf{75.43_{\pm 0.34}}$} 
& $38.96_{\pm 0.33}$ & $49.19_{\pm 0.58}$ 
& $77.31_{\pm 0.29}$ & $73.14_{\pm 2.52}$ \\

LGMRec 
& $80.06_{\pm 0.39}$ & $73.21_{\pm 0.49}$ 
& $39.55_{\pm 0.27}$ & $47.62_{\pm 0.51}$ 
& $73.41_{\pm 0.61}$ & $74.02_{\pm 3.34}$ \\

\rowcolor{gray!10}
MLaGA 
& $81.52_{\pm 0.36}$ & $74.83_{\pm 0.32}$ 
& $39.14_{\pm 0.88}$ & $49.82_{\pm 1.34}$ 
& $77.69_{\pm 0.33}$ & $73.12_{\pm 2.64}$ \\

GraphGPT-O 
& $78.27_{\pm 0.21}$ & $70.72_{\pm 0.55}$ 
& $37.22_{\pm 0.48}$ & $46.54_{\pm 0.24}$ 
& $75.42_{\pm 0.61}$ & $71.15_{\pm 3.04}$ \\

\rowcolor{gray!10}
Graph4MM 
& $79.36_{\pm 0.24}$ & $72.15_{\pm 0.28}$ 
& $38.48_{\pm 0.24}$ & $47.45_{\pm 0.31}$ 
& $75.28_{\pm 0.14}$ & $70.47_{\pm 1.88}$ \\

InstructG2I 
& $80.42_{\pm 0.28}$ & $73.54_{\pm 0.20}$ 
& $39.12_{\pm 0.39}$ & $49.83_{\pm 0.27}$ 
& $75.36_{\pm 0.52}$ & $71.42_{\pm 2.18}$ \\

\rowcolor{gray!10}
DMGC 
& $81.55_{\pm 0.40}$ & $71.32_{\pm 0.48}$ 
& \textcolor{royalblue}{$\mathbf{41.37_{\pm 0.26}}$}
& \textcolor{royalblue}{$\mathbf{55.63_{\pm 0.19}}$} 
& \textcolor{royalblue}{$\mathbf{79.31_{\pm 0.32}}$} 
& \textcolor{royalblue}{$\mathbf{75.10_{\pm 3.14}}$} \\

DGF 
& \textcolor{orange}{$\mathbf{82.22_{\pm 0.15}}$} 
& $74.24_{\pm 0.57}$ 
& \textcolor{orange}{$\mathbf{39.80_{\pm 0.32}}$} 
& \textcolor{orange}{$\mathbf{50.10_{\pm 0.35}}$} 
& \textcolor{orange}{$\mathbf{78.50_{\pm 0.30}}$} 
& \textcolor{orange}{$\mathbf{74.47_{\pm 2.55}}$} \\

\rowcolor{gray!10}
MIG-GT 
& $80.88_{\pm 0.41}$ 
& $72.26_{\pm 0.39}$ 
& $37.82_{\pm 0.44}$ 
& $48.31_{\pm 0.69}$ 
& $76.34_{\pm 0.42}$ 
& $71.58_{\pm 2.73}$ \\

NTSFormer 
& $81.85_{\pm 0.30}$ 
& $74.10_{\pm 0.32}$ 
& $38.73_{\pm 0.34}$ 
& $49.27_{\pm 0.49}$ 
& $77.25_{\pm 0.32}$ 
& $72.81_{\pm 2.46}$ \\

\rowcolor{gray!10}
UniGraph2 
& $80.52_{\pm 0.36}$ 
& $72.19_{\pm 0.43}$ 
& $37.47_{\pm 0.32}$ 
& $47.78_{\pm 0.38}$ 
& $75.92_{\pm 0.44}$ 
& $71.84_{\pm 2.58}$ \\

\hline

TMTE (Ours) 
& \textcolor{darkred}{$\mathbf{84.18_{\pm 0.30}}$} 
& \textcolor{darkred}{$\mathbf{80.80_{\pm 0.45}}$} 
& \textcolor{darkred}{$\mathbf{43.61_{\pm 0.24}}$} 
& \textcolor{darkred}{$\mathbf{60.61_{\pm 0.31}}$} 
& \textcolor{darkred}{$\mathbf{82.28_{\pm 0.21}}$} 
& \textcolor{darkred}{$\mathbf{78.59_{\pm 1.44}}$} \\

\hline\thickhline
\end{tabular}
}}
\end{table*}

We provide evaluations on additional MAG datasets in Table~\ref{tab: graph_centric_more}. As observed, TMTE consistently outperforms all baselines on the Grocery, Bili Dance, and RedditS datasets across node classification, link prediction, and node clustering metrics. Specifically, for node classification on Grocery, TMTE improves over the second-best method by +1.96\% in Accuracy and +6.36\% in F1-score. For link prediction on Bili Dance, it achieves gains of +2.24\% in MRR and +5.98\% in Hits@3. In node clustering on RedditS, TMTE boosts NMI and ARI by +3.19\% and +3.49\% compared with the next-best approach. These results highlight that, while some multimodal graph methods (e.g., DGF, DMGC) outperform unimodal GNNs on certain tasks, TMTE demonstrates stable and consistent superiority across all datasets and evaluation metrics.

\section{Theoretical Proofs}
\label{appendix: proofs}

\begin{theorem}[Smooth Fused Representations of MAG]
\label{thm: smooth fused rep}
For a MAG with fused representation $\bar{\mathbf{H}}=\frac{1}{\mathcal{M}} \sum_{m\in\mathcal{M}}\mathbf{H}^{(m)}$ and a symmetrically normalized adjacency matrix of evolved topology 
$\mathbf{Q}^{E_1}=\lambda\,\tilde{\mathbf{A}} + (1-\lambda)\mathbf{A}^{E_1}$, the smooth fused representations can be approximated as:
\begin{equation}
    \hat{\mathbf{H}}
    =
    \frac{1}{\alpha+1}
    \sum_{t=0}^{T}
    \left(
    \frac{\alpha\lambda}{\alpha+1}\tilde{\mathbf{A}}
    +
    \frac{\alpha(1-\lambda)}{\alpha+1}\mathbf{A}^{E_1}
    \right)^t
    \bar{\mathbf{H}}.
\end{equation}
\end{theorem}

\begin{proof} 
The fused representation that maintains smoothness over the evolved topology is obtained by minimizing:
\begin{equation}
    \label{proof_eq_smooth}
    {\hat{\mathbf{H}}} = \argmin_{\hat{\mathbf{H}}} f(\hat{\mathbf{H}})=\argmin_{\hat{\mathbf{H}}} \|\hat{\mathbf{H}} - \bar{\mathbf{H}}\|_F^2 + \alpha \cdot \text{tr}\big(\hat{\mathbf{H}}^\top(\mathbf{I}-\mathbf{Q}^{E_1})\hat{\mathbf{H}}\big),
\end{equation}
where $\alpha \in (0,1)$ controls the smoothness strength, and the Laplacian regularizer $\mathbf{I}-\mathbf{Q}^{E_1}$ encourages neighboring nodes to have similar embeddings. The fidelity term $\|\hat{\mathbf{H}} - \bar{\mathbf{H}}\|_F^2$ ensures that the learned embeddings remain close to the initial fused representation.  

Setting $\frac{\partial f(\hat{\mathbf{H}})}{\partial \hat{\mathbf{H}}}=0$ yields:
\begin{equation}
    {\hat{\mathbf{H}}} = \frac{1}{\alpha+1} (\mathbf{I}-\frac{\alpha}{\alpha+1} \mathbf{Q}^{E_1})^{-1}\bar{\mathbf{H}}.
\end{equation}

Notably, $\mathbf{Q}^{E_1}$ is a linear combination of two symmetrically normalized adjacency matrices, $\tilde{\mathbf{A}}$ and $\mathbf{A}^{E_1}$, with weights summing to 1. Both matrices are symmetric, and their eigenvalues are real and bounded by 1. The maximum eigenvalue of a symmetric matrix can be characterized by the Rayleigh quotient:
\[
\lambda_{\max}(\mathbf{Q}^{E_1}) = \max_{x \neq 0} \frac{x^\top \mathbf{Q}^{E_1} x}{x^\top x}.
\]

By convexity of the maximum Rayleigh quotient over symmetric matrices, we have
\[
\lambda_{\max}(\mathbf{Q}^{E_1}) \le \lambda \lambda_{\max}(\tilde{\mathbf{A}}) + (1-\lambda) \lambda_{\max}(\mathbf{A}^{E_1}) \le 1.
\]

Thus, the dominant eigenvalue of $\mathbf{Q}^{E_1}$ is guaranteed to be $\le 1$, ensuring stability in spectral-based operations such as graph convolution. As shown in~\cite{matrix_analysis}, when the dominant eigenvalue is less than 1, the inverse $(\mathbf{I}-\mathbf{Q}^{E_1})^{-1}$ can be expressed as a Neumann series:
\begin{equation}
    (\mathbf{I}-\mathbf{Q}^{E_1})^{-1} = \sum_{t=0}^\infty {\mathbf{Q}^{E_1}}^{t}.
\end{equation}

Truncating the series at $T$ steps gives the power series approximation:
\begin{equation}
    \hat{\mathbf{H}}
    =
    \frac{1}{\alpha+1}
    \sum_{t=0}^{T}
    \left(
    \frac{\alpha\lambda}{\alpha+1}\tilde{\mathbf{A}}
    +
    \frac{\alpha(1-\lambda)}{\alpha+1}\mathbf{A}^{E_1}
    \right)^t
    \bar{\mathbf{H}}.
\end{equation}
\end{proof}

\begin{theorem}[Recursive Power Expansion of Smooth Fused Representations]
\label{thm: recursive power expansion}
For the smooth fused representation defined in Eq.~\eqref{eq: approximated}, let $\mathbf{H}^{(0)} = \bar{\mathbf{H}}$ and define the recursive sequence
$
\mathbf{H}^{(t+1)}=
\frac{\alpha\lambda}{\alpha+1}
\tilde{\mathbf{A}}
\mathbf{H}^{(t)}
+
\frac{\alpha(1-\lambda)}{\alpha+1}
\mathbf{A}^{E_1}
\mathbf{H}^{(t)}.
$
Then for any $T \ge 0$, the truncated power series satisfies
$
\hat{\mathbf{H}}=
\frac{1}{\alpha+1}
\sum_{t=0}^{T}
\mathbf{H}^{(t)},
$
and each term $\mathbf{H}^{(t)}$ can be recursively computed without explicitly forming $\mathbf{A}^{E_1}$.
\end{theorem}

\begin{proof}

Recall from Eq.~\eqref{eq: approximated} that
\begin{equation}
\hat{\mathbf{H}}=
\frac{1}{\alpha+1}
\sum_{t=0}^{T}
\left(
\frac{\alpha\lambda}{\alpha+1}
\tilde{\mathbf{A}}
+
\frac{\alpha(1-\lambda)}{\alpha+1}
\mathbf{A}^{E_1}
\right)^{t}
\bar{\mathbf{H}}.
\end{equation}

Define the linear operator
\begin{equation}
\mathbf{S}=
\frac{\alpha\lambda}{\alpha+1}
\tilde{\mathbf{A}}
+
\frac{\alpha(1-\lambda)}{\alpha+1}
\mathbf{A}^{E_1}.
\end{equation}
Let
$
\mathbf{H}^{(0)}=
\bar{\mathbf{H}},
$
and recursively define
$
\mathbf{H}^{(t+1)}=
\mathbf{S}
\mathbf{H}^{(t)}.
$
We prove by induction that
\begin{equation}
\mathbf{H}^{(t)}=
\mathbf{S}^{t}
\bar{\mathbf{H}}
\quad
\text{for all } t \ge 0.
\end{equation}

\textbf{Base case:} For $t = 0$,
\begin{equation}
\mathbf{H}^{(0)}=
\bar{\mathbf{H}}=
\mathbf{S}^{0}
\bar{\mathbf{H}}.
\end{equation}

\textbf{Inductive step:} Assume that
\begin{equation}
\mathbf{H}^{(t)}=
\mathbf{S}^{t}
\bar{\mathbf{H}}.
\end{equation}

Then

\begin{equation}
\mathbf{H}^{(t+1)}=
\mathbf{S}
\mathbf{H}^{(t)}=
\mathbf{S}(
\mathbf{S}^{t}
\bar{\mathbf{H}})=
\mathbf{S}^{t+1}
\bar{\mathbf{H}}.
\end{equation}
Thus, by mathematical induction,
\begin{equation}
\mathbf{H}^{(t)}=
\mathbf{S}^{t}
\bar{\mathbf{H}}
\quad
\forall t \in {0,1,\dots,T}.
\end{equation}
Substituting this result back into the truncated series gives
\begin{equation}
\hat{\mathbf{H}}=
\frac{1}{\alpha+1}
\sum_{t=0}^{T}
\mathbf{H}^{(t)}.
\end{equation}
Next, we show that each recursive step can be expanded without constructing $\mathbf{A}^{E_1}$. By definition,
$
\mathbf{A}^{E_1}=
{\mathbf{\Delta}^{E_1}}^{-1}
\mathbf{R}^{E_1}
{\mathbf{\Lambda}^{E_1}}^{-1}
{\mathbf{R}^{E_1}}^\top.
$
Thus, for any $\mathbf{H}^{(t)}$,
$
\mathbf{A}^{E_1}
\mathbf{H}^{(t)}=
{\mathbf{\Delta}^{E_1}}^{-1}
\mathbf{R}^{E_1}
{\mathbf{\Lambda}^{E_1}}^{-1}
{\mathbf{R}^{E_1}}^\top
\mathbf{H}^{(t)}.
$
Using associativity of matrix multiplication, we regroup terms:
\begin{equation}
\mathbf{A}^{E_1}
\mathbf{H}^{(t)}=
{\mathbf{\Delta}^{E_1}}^{-1}
\mathbf{R}^{E_1}
\left(
{\mathbf{\Lambda}^{E_1}}^{-1}
{\mathbf{R}^{E_1}}^\top
\mathbf{H}^{(t)}
\right).
\end{equation}
Define
$
\mathbf{Z}_{\mathcal{U}}^{(t)}=
{\mathbf{\Lambda}^{E_1}}^{-1}
{\mathbf{R}^{E_1}}^\top
\mathbf{H}^{(t)},
$
and
$
\mathbf{Z}_{\mathcal{V}}^{(t)}=
{\mathbf{\Delta}^{E_1}}^{-1}
\mathbf{R}^{E_1}
\mathbf{Z}_{\mathcal{U}}^{(t)}.
$
Then
\begin{equation}
\mathbf{Z}_{\mathcal{V}}^{(t)}=
\mathbf{A}^{E_1}
\mathbf{H}^{(t)}.
\end{equation}
Therefore, for every $t = 0,1,\dots,T-1$, the recursive update
\begin{equation}
\mathbf{H}^{(t+1)}=
\frac{\alpha\lambda}{\alpha+1}
\tilde{\mathbf{A}}
\mathbf{H}^{(t)}
+
\frac{\alpha(1-\lambda)}{\alpha+1}
\mathbf{Z}_{\mathcal{V}}^{(t)}
\end{equation}
is algebraically equivalent to
\begin{equation}
\mathbf{H}^{(t+1)}=
\mathbf{S}
\mathbf{H}^{(t)},
\end{equation}
and thus exactly reproduces the truncated power expansion up to order $T$.
\end{proof}

\begin{theorem}[Stability under inexact large-scale propagation]
\label{thm: inexact stability}
Assume the practical recursion on large-scale graphs is computed inexactly:
\begin{equation}
\widetilde{\mathbf{H}}^{(t+1)}=\mathbf{S}\widetilde{\mathbf{H}}^{(t)}+\mathbf{E}^{(t)},
\quad
\widetilde{\mathbf{H}}^{(0)}=\bar{\mathbf{H}},
\end{equation}
where $\|\mathbf{E}^{(t)}\|_F\le\varepsilon_t$. Under the assumptions of Theorem~\ref{thm: contraction large scale},
\begin{equation}
\|\widetilde{\mathbf{H}}^{(t)}-\mathbf{H}^{(t)}\|_F
\le
\sum_{k=0}^{t-1}\beta^{t-1-k}\varepsilon_k.
\end{equation}
In particular, if $\varepsilon_k\le\bar\varepsilon$ for all $k$, then
\begin{equation}
\sup_t\|\widetilde{\mathbf{H}}^{(t)}-\mathbf{H}^{(t)}\|_F
\le
\frac{\bar\varepsilon}{1-\beta},
\end{equation}
which gives a uniform stability guarantee.
\end{theorem}

\begin{proof}
Let $\mathbf{D}^{(t)}:=\widetilde{\mathbf{H}}^{(t)}-\mathbf{H}^{(t)}$. Then
\begin{equation}
\mathbf{D}^{(t+1)}=\mathbf{S}\mathbf{D}^{(t)}+\mathbf{E}^{(t)},
\quad
\mathbf{D}^{(0)}=\mathbf{0}.
\end{equation}
Unrolling recursion yields
\begin{equation}
\mathbf{D}^{(t)}=\sum_{k=0}^{t-1}\mathbf{S}^{t-1-k}\mathbf{E}^{(k)}.
\end{equation}
Therefore,
\begin{equation}
\|\mathbf{D}^{(t)}\|_F
\le
\sum_{k=0}^{t-1}\|\mathbf{S}\|_2^{t-1-k}\|\mathbf{E}^{(k)}\|_F
\le
\sum_{k=0}^{t-1}\beta^{t-1-k}\varepsilon_k.
\end{equation}
If $\varepsilon_k\le\bar\varepsilon$, then
\begin{equation}
\|\mathbf{D}^{(t)}\|_F
\le
\bar\varepsilon\sum_{j=0}^{t-1}\beta^j
\le
\frac{\bar\varepsilon}{1-\beta}.
\end{equation}
\end{proof}

\begin{theorem}[Contraction, uniqueness, and convergence rate on large-scale graphs]
\label{thm: contraction large scale}
Let
\begin{equation}
\mathbf{S}=
\frac{\alpha\lambda}{\alpha+1}\tilde{\mathbf{A}}
+
\frac{\alpha(1-\lambda)}{\alpha+1}\mathbf{A}^{E_1},
\quad
\beta:=\frac{\alpha}{\alpha+1}\in(0,1),
\end{equation}
and assume $\|\tilde{\mathbf{A}}\|_2\le 1$ and $\|\mathbf{A}^{E_1}\|_2\le 1$ (true for symmetric normalized adjacency operators). Then:
\begin{equation}
\|\mathbf{S}\|_2\le \beta<1,
\end{equation}
so $(\mathbf{I}-\mathbf{S})$ is invertible and the smooth fused representation
\begin{equation}
\hat{\mathbf{H}}^\star=\frac{1}{\alpha+1}(\mathbf{I}-\mathbf{S})^{-1}\bar{\mathbf{H}}
\end{equation}
is unique. Moreover, for the recursion $\mathbf{H}^{(t+1)}=\mathbf{S}\mathbf{H}^{(t)}$ with $\mathbf{H}^{(0)}=\bar{\mathbf{H}}$,
\begin{equation}
\left\|\hat{\mathbf{H}}^\star-\frac{1}{\alpha+1}\sum_{t=0}^{T}\mathbf{H}^{(t)}\right\|_F
\le
\frac{\beta^{T+1}}{(\alpha+1)(1-\beta)}\,\|\bar{\mathbf{H}}\|_F,
\end{equation}
i.e., the truncation error decays geometrically.
\end{theorem}

\begin{proof}
By triangle inequality and sub-multiplicativity,
\begin{equation}
\|\mathbf{S}\|_2
\le
\frac{\alpha\lambda}{\alpha+1}\|\tilde{\mathbf{A}}\|_2
+
\frac{\alpha(1-\lambda)}{\alpha+1}\|\mathbf{A}^{E_1}\|_2
\le
\frac{\alpha}{\alpha+1}=\beta<1.
\end{equation}
Hence $\rho(\mathbf{S})\le\|\mathbf{S}\|_2<1$, so Neumann series applies:
\begin{equation}
(\mathbf{I}-\mathbf{S})^{-1}=\sum_{t=0}^{\infty}\mathbf{S}^t,
\end{equation}
and $\hat{\mathbf{H}}^\star$ is unique. Also,
\begin{equation}
\hat{\mathbf{H}}^\star
-
\frac{1}{\alpha+1}\sum_{t=0}^{T}\mathbf{H}^{(t)}
=
\frac{1}{\alpha+1}\sum_{t=T+1}^{\infty}\mathbf{S}^t\bar{\mathbf{H}}.
\end{equation}
Taking Frobenius norm and using $\|\mathbf{S}^t\bar{\mathbf{H}}\|_F\le\|\mathbf{S}\|_2^t\|\bar{\mathbf{H}}\|_F\le\beta^t\|\bar{\mathbf{H}}\|_F$,
\begin{equation}
\left\|\hat{\mathbf{H}}^\star-\frac{1}{\alpha+1}\sum_{t=0}^{T}\mathbf{H}^{(t)}\right\|_F
\le
\frac{1}{\alpha+1}\sum_{t=T+1}^{\infty}\beta^t\|\bar{\mathbf{H}}\|_F
=
\frac{\beta^{T+1}}{(\alpha+1)(1-\beta)}\|\bar{\mathbf{H}}\|_F.
\end{equation}
\end{proof}

\section{Pseudocode of \textsc{TMTE}}
\label{appendix: pseudo code}

We provide the detailed procedure for TMTE in Algorithm~\ref{alg: tmte}.

\normalem
\begin{algorithm}[htbp]\fontsize{8pt}{4pt}\selectfont
\DontPrintSemicolon
\SetAlgoLined
\caption{Overall Procedure of TMTE}
\label{alg: tmte}

\KwInput{
MAG $\mathcal{G}=(\mathcal{V},\mathcal{E},\{\mathbf{X}^{(m)}\}_{m\in\mathcal{M}})$; 
anchor size $|\mathcal{U}|$; 
maximum epochs $E$; 
threshold $\delta$.
}

\KwOutput{Fused embedding $\hat{\mathbf{H}}$ and modality embeddings $\{\mathbf{H}^{(m)}\}$.}

\tcc{Initialization}
initialize model parameters; \\
compute normalized original adjacency $\tilde{\mathbf{A}}$; \\

\For{$e=1$ to $E$}{

    \tcc{Anchor Resampling}
    randomly sample anchor set $\mathcal{U}\subseteq\mathcal{V}$; \\

    \tcc{Step 1: Topology Evolution from Original Modality Space}
    compute node-anchor affinity matrix $\mathbf{R}^{E_1}$ 
    using $\{\mathbf{X}^{(m)}\}$ via Eqs.~(\ref{eq: sim channel}–\ref{eq: sim combine}); \\
    construct implicit topology 
    $\mathbf{Q}^{E_1}=\lambda\tilde{\mathbf{A}}+(1-\lambda)\mathbf{A}^{E_1}$; \\

    \tcc{Step 2: Modality Evolution}
    compute modality embeddings 
    $\{\mathbf{H}^{(m)}\}$ via Eq.~(\ref{eq: modality embedding}); \\
    compute fused embedding 
    $\hat{\mathbf{H}}^{E_1}$ via Eq.~(\ref{eq: approximated}) using $\mathbf{Q}^{E_1}$; \\

    \tcc{Step 3: Task-aware Co-evolution}
    set $k=2$; \\
    \While{$k\leqslant T$ \textbf{and} 
    $\frac{\|\mathbf{R}^{E_k}-\mathbf{R}^{E_{k-1}}\|_F^2}
    {\|\mathbf{R}^{E_k}\|_F^2} > \delta$}{

        update node-anchor affinity 
        $\mathbf{R}^{E_{k+1}}$ 
        using $\{\mathbf{H}^{(m)}\}$; \\

        update topology 
        $\mathbf{Q}^{E_{k+1}}$; \\

        update modality embeddings 
        $\mathbf{H}^{(m)}$; \\

        update fused embedding 
        $\hat{\mathbf{H}}^{E_{k+1}}$; \\

        optimize objective 
        $\mathcal{L}=\mathcal{L}_{mod}+\eta\mathcal{L}_{task}$; \\

        $k \leftarrow k+1$; \\
    }

}

\Return{$\hat{\mathbf{H}}, \{\mathbf{H}^{(m)}\}$};

\end{algorithm}
\ULforem

\section{Detailed Experimental Setups}
\label{appendix: more experimental setups}

\vspace{+0.1cm}
\noindent \textbf{Empirical Study Details.} For the \textit{Modality-optimized Topology} setting, we compute cross-modality similarity between two nodes on the Toys dataset and select the top-5 similar nodes to connect edges, which can provide semantic-related knowledge for the G2Image task.

\vspace{+0.1cm}
\noindent \textbf{Task Objective Details.} As TMTE is designed to leverage the downstream task optimization objective to guide the co-evolution between topology and modality (Eq.~\eqref{eq: loss_tmte}), we select different loss functions $\mathcal{L}_\text{loss}$ for different tasks. Specifically, for node classification, we adopt the cross-entropy loss; for link prediction, we adopt the binary cross-entropy loss; for node clustering, we adopt the community-aware (i.e., cluster) cross-modality contrastive loss proposed in DGF~\cite{dgf}. For the three modality-centric tasks, we adopt the cross-modality contrastive loss as the loss function.

\vspace{+0.1cm}
\noindent \textbf{Algorithm Hyperparameters.} For our proposed TMTE, we fix the topology evolution trade-off $\lambda$ to $0.8$ for all datasets and tasks. The number of perspectives $K$ is searched within $[1, 2, 4, 8, 16, 32]$, the hidden dimension of the similarity metric learning is searched within $[32, 64, 128, 256, 512]$, smoothness rate $\alpha$ is searched within $[10^{-2}, 10^{-1}, 1, 10, 100]$, the maximum evolution rounds are fixed to $10$, the stop-evolution threshold $\eta$ is searched within $[1\times 10^{-6}, 1\times 10^{-5}]$, the truncation step $T$ is fixed to $10$, and the modality evolution trade-off $\eta$ is searched within $[10^{-3}, 10^{-2}, 10^{-1}]$. For the baselines, we adopt the hyperparameter configurations reported in their original papers whenever available. When unspecified, we employ automated hyperparameter optimization using the Optuna framework~\cite{akiba2019optuna}. 

\vspace{+0.1cm}
\noindent \textbf{Task Hyperparameters and Metrics.} To ensure a fair comparison across diverse MAG learning models, we employ unified experimental protocols for three graph-centric downstream tasks: supervised node classification, supervised link prediction, and unsupervised node clustering. For node classification and clustering, we set the learning rate to $5 \times 10^{-3}$, with a batch size of 512 and weight decay of $1 \times 10^{-5}$. Node classification is trained for 100 epochs, which is sufficient for convergence, while node clustering is optimized for 500 epochs to stabilize the self-supervised objective. For link prediction, we adopt a learning rate of $1 \times 10^{-3}$ and increase the batch size to 2048 to efficiently accommodate large-scale edge-pair samples. To ensure architectural consistency, we use Qwen2-VL-7B-Instruct~\cite{bai2023qwen} as a frozen feature encoder and fix the feature dimensionality to 768 across all tasks. Each experiment is repeated three times with different random seeds, and we report the mean performance to mitigate variance due to initialization.

\vspace{+0.1cm}
\noindent \textbf{Node Classification} is a supervised task where each node in a MAG is encoded into a low-dimensional embedding, followed by a projection head and a Softmax layer to produce class probabilities. The model is optimized via cross-entropy loss against ground-truth labels. Performance is evaluated using Accuracy (Acc) and F1-score. Specifically, $\text{Acc} = \frac{1}{N} \sum_{i=1}^{N} \mathbb{I}(\hat{y}_i = y_i)$, where $N$ is the sample size, $y_i$ the ground-truth label, $\hat{y}_i$ the prediction, and $\mathbb{I}(\cdot)$ the indicator function. The F1-score, defined as $\text{F1} = 2 \left(\cdot \text{Precision} \cdot \text{Recall}\right)/\left(\text{Precision} + \text{Recall}\right)$, captures robustness under class imbalance.

\vspace{+0.1cm}
\noindent \textbf{Link Prediction} evaluates the ability to infer missing or potential edges. The model assigns similarity scores (e.g., dot product) to node pairs, encouraging higher scores for positive edges than negatives. In multimodal graphs, this requires aligning structural proximity with cross-modal semantics. We adopt ranking-based metrics: Mean Reciprocal Rank (MRR) and Hits@K. Formally, $\text{MRR} = \frac{1}{|Q|} \sum_{i=1}^{|Q|} \frac{1}{\text{rank}_i}$, where $\text{rank}_i$ is the rank of the first correct target for query $i$. $\text{Hits@K} = \frac{1}{|Q|} \sum_{i=1}^{|Q|} \mathbb{I}(\text{rank}_i \leq K)$ measures recall at cutoff $K$.

\vspace{+0.1cm}
\noindent \textbf{Node Clustering} assesses representation quality in an unsupervised setting following \cite{dmgc}. The model disentangles homophilous and heterophilous views, fuses dual-frequency signals, and is trained with a joint objective comprising reconstruction, contrastive, and clustering losses. Performance is measured by NMI and ARI. $\text{NMI}(Y, C) = 2 \cdot I(Y, C)/\left(H(Y) + H(C)\right)$ quantifies mutual dependence between predicted clusters $C$ and ground-truth labels $Y$. ARI is computed as:
\begin{equation}
    \begin{aligned}
        \text{ARI} = \frac{\sum_{ij} \binom{n_{ij}}{2} - [\sum_i \binom{a_i}{2} \sum_j \binom{b_j}{2}] / \binom{n}{2}}{ \frac{1}{2} [\sum_i \binom{a_i}{2} + \sum_j \binom{b_j}{2}] - [\sum_i \binom{a_i}{2} \sum_j \binom{b_j}{2}] / \binom{n}{2} },
    \end{aligned}
\end{equation}
where $n_{ij}$ denotes the overlap between ground-truth cluster $i$ and predicted cluster $j$.

\vspace{+0.1cm}
\noindent \textbf{Modality Retrieval} projects queries and candidates from different modalities into a shared latent space and ranks them by similarity. We optimize a contrastive objective with temperature $\tau = 0.07$, train for 500 epochs using learning rate $1 \times 10^{-3}$ and batch size 256, and apply early stopping with patience 10--25 epochs. Evaluation uses MRR and Hits@K.

\vspace{+0.1cm}
\noindent \textbf{Graph-to-Text (G2Text)} generates textual descriptions conditioned on multimodal graph inputs. Following MMGL~\cite{yoon2023mmgl}, the pipeline includes: (1) neighbor encoding into a unified embedding space; (2) graph structure encoding via GNNs or Laplacian positional encodings; and (3) integration into a pre-trained LLM. We train with learning rate $1 \times 10^{-3}$, weight decay $1 \times 10^{-2}$, batch size 8 for 15 epochs. SA-E samples four multimodal neighbors, and GNNs encode structure. The decoder backbone is Facebook OPT-125M, adapted via Prefix Tuning or LoRA ($r=64$). Performance is measured by BLEU-4, ROUGE-L, and CIDEr. Specifically, BLEU-4 evaluates lexical accuracy and fluency by calculating the geometric mean of modified $n$-gram precisions ($p_n$) up to length 4: 
    \begin{equation}
        \begin{aligned}
            \text{BLEU-4} = \text{BP} \cdot \exp\left(\sum_{n=1}^4 w_n \log p_n\right).
        \end{aligned}
    \end{equation}
    ROUGE-L measures sentence-level recall based on the longest common subsequence, ensuring the output covers the comprehensive information content of the ground truth: 
    \begin{equation}
        \begin{aligned}
            \text{ROUGE-L} = \frac{(1 + \beta^2) R_{lcs} P_{lcs}}{R_{lcs} + \beta^2 P_{lcs}}.
        \end{aligned}
    \end{equation}
    CIDEr measures the consensus between generated captions and human references using TF-IDF weighting, emphasizing the semantic importance and distinctiveness of the generated terms:
    \begin{equation}
    \begin{aligned}
        \text{CIDEr}(c, r) = \frac{1}{M} \sum_{i=1}^{M} \frac{g^n(c) \cdot g^n(r_i)}{|g^n(c)| |g^n(r_i)|}.
    \end{aligned}
    \end{equation}

\vspace{+0.1cm}
\noindent \textbf{Graph-to-Image (G2Image)} synthesizes images conditioned on MAG. Following InstructG2I~\cite{instructg2i}, we adopt: (1) semantic PPR-based sampling (0--6 neighbors); (2) Graph-QFormer encoding; and (3) latent diffusion with graph classifier-free guidance. Training uses learning rate $1 \times 10^{-4}$, batch size 16 for 20 epochs, and image resolution 256. Evaluation employs CLIP-Score and DINOv2-Score. Specifically, CLIP-Score quantifies cross-modal semantic consistency. 
    Building upon the textual encoder ($E_T$) and visual encoder ($E_I$) in pre-trained CLIP, this metric determines whether generated images faithfully preserve the semantic content of corresponding graph descriptions:
    \begin{equation}
    \begin{aligned}
        \text{CLIP-Score}(I, T) = \max\left(100 \cdot \cos(E_I(I), E_T(T)), 0\right).
    \end{aligned}
    \end{equation}
    DINOv2-Score assesses visual fidelity and structural consistency using feature embeddings from a pre-trained DINOv2 encoder, ensuring high perceptual quality and structural resemblance to reference samples:
    \begin{equation}
    \begin{aligned}
        \text{DINOv2-Score}(I_{\text{gen}}, I_{\text{ref}}) = \cos(\text{DINO}(I_{\text{gen}}), \text{DINO}(I_{\text{ref}})).
    \end{aligned}
    \end{equation}

\section{Experimental Environment}
\label{appendix: environment}

Experiments are conducted on a workstation equipped with Intel Xeon Scalable processors and NVIDIA RTX 6000 Ada Generation GPUs with 96 GB of VRAM, supported by 256 GB of system RAM. The computational environment utilizes CUDA 12.9, while software implementations are developed using Python 3.10.18 and PyTorch 2.8.

\section{Dataset Details}
\label{appendix: dataset details}

Detailed statistical information on the datasets is presented in Table~\ref{table: datasets}, with textual descriptions as follows.
\begin{table*}[htbp]
\centering
\caption{The statistical information of the experimental datasets.}
\label{table: datasets}
\begin{tabular}{c|cccc|cc}
\midrule[0.3pt]
Datasets     & \# Modalities & \# Nodes & \# Edges   & \# Classes  & Domain    \\ \midrule[0.3pt]

Movies       & Text, Image  & 16,672  & 218,390   & 20               & Movie Network  \\

Grocery      & Text, Image  & 17,074  & 171,340   & 20             & E-Commerce Network\\

DY & Text, Image & 8,299 & 35,627 & - & Video Network\\

Bili Dance & Text, Image & 2,307 & 9,127 & - & Video Network\\

Toys & Text, Image & 20,695 & 126,886 & 18 & E-Commerce Network\\

RedditS      & Text, Image  & 15,894  & 566,160   & 20               & Social Network \\

Ele-fashion  & Text, Image  & 97,766  & 199,602   & 12  & Co-purchase Network \\
Flickr30k    & Text, Image  & 31,783  & 181,151   & -                & Image Network  \\
SemArt       & Text, Image  & 21,382  & 1,216,432 & - & Art Network    \\
MVSA & Text, Image  & 2,122  & 17,790 & 2 & KNN Sentiment Graph     \\


\midrule[0.3pt]
\end{tabular}
\vspace{0.3cm}
\end{table*}

\vspace{+0.1cm}
\noindent \textbf{Movies}~\cite{ni2019justifying} 
    is sourced from Amazon’s Movies and TV category. 
    Nodes correspond to DVD/Blu-ray products, and edges reflect consumer co-purchasing behavior. 
    Node attributes include textual plot synopses and customer reviews, alongside visual features derived from official cover art. 

\vspace{+0.1cm}
\noindent \textbf{Grocery}~\cite{ni2019justifying} is sourced from Amazon’s Grocery and Gourmet Food category. Nodes are food and household
products, and edges indicate complementary purchasing habits derived from ”also-bought” lists. Textual attributes are
encoded from product titles and nutritional descriptions, while visual attributes are extracted from packaging images. This dataset is primarily utilized for Node Classification, where labels correspond to fine-grained product sub-categories (e.g.,
Beverages, Snacks).

\vspace{+0.1cm}
\noindent \textbf{DY}~\cite{zhang2024ninerec} is sourced from Douyin, a leading short-video platform. The graph contains short-video
nodes linked by co-interaction patterns (e.g., liked by the same user). Given the multimodal nature of short videos, textual
features are constructed from user captions and hashtags, while visual features are extracted from video frames, capturing
the fast-paced visual dynamics typical of the platform. Preprocessing involved filtering low-frequency items to maintain
graph connectivity. This dataset is primarily utilized for Link Prediction tasks.

\vspace{+0.1cm}
\noindent \textbf{Bili Dance}~\cite{zhang2024ninerec} originates from the dance section of Bilibili. Nodes represent dance performance or tutorial
videos. Co-viewing edges capture the trend similarity or sequential learning patterns of viewers (e.g., users watching
consecutive tutorials). Textual attributes are derived from dense video descriptions and hashtags, while visual attributes are
obtained from keyframes of the dance movements to encode choreographic dynamics. This dataset is primarily utilized for
Link Prediction tasks.

\vspace{+0.1cm}
\noindent \textbf{Toys}~\cite{ni2019justifying} originates from Amazon’s Toys and Games category. The graph connects toy product nodes via
co-purchasing edges. Textual features are derived from product specifications and age recommendations, and visual features
are obtained from product photos to identify visual variants. This dataset is primarily utilized for Node Classification, where
labels represent specific toy types or game genres.

\vspace{+0.1cm}
\noindent \textbf{RedditS}~\cite{RedditS}
    is a social network from Reddit where nodes represent posts and edges denote threading relationships (e.g., comments and replies). 
    Textual features are encoded from titles and body content, while visual features are extracted from embedded images.
    This dataset is primarily used for node classification and node clustering.


\vspace{+0.1cm}
\noindent \textbf{Ele-fashion}~\cite{ni2019justifying, hou2024bridging} 
    is a heterogeneous graph merging Amazon’s Electronics and Fashion categories. Nodes are connected via cross-category co-purchasing links, revealing latent consumer preferences across disparate domains.
    Features combine technical specs with style descriptions and product imagery.

\vspace{+0.1cm}
\noindent \textbf{Flickr30k}~\cite{plummer2015_Flickr30k} 
    is a canonical image-text reasoning dataset. 
    In the OpenMAG setting, we construct a graph where nodes represent image regions and caption phrases, linked by semantic grounding annotations. 
    This dataset is utilized for Graph-to-Text (G2Text) tasks, evaluating the model's ability to generate descriptive captions by traversing grounded visual-textual relationships.

\vspace{+0.1cm}
\noindent \textbf{SemArt}~\cite{garcia2018_SemArt} 
    is a fine-art dataset where nodes represent paintings and edges are established based on shared metadata such as artist, period, or school. 
    Node features include expert historical commentary and stylistic visual attributes from digital images. 
    SemArt serves as a benchmark for Graph-to-Image (G2Image) tasks, requiring the reconciliation of abstract historical descriptions with complex visual aesthetics.


\vspace{+0.1cm}
\noindent \textbf{MVSA}~\cite{MVSA} is a multi-view sentiment analysis dataset where each sample consists of a paired image and its corresponding tweet text collected from Twitter. The dataset is designed to facilitate research on multimodal sentiment classification, where the goal is to predict users’ attitudes (e.g., positive, negative, or neutral) by jointly modeling visual and textual information. Notably, since the sentiment labels corresponding to images and text may differ, we only retain samples where both the image and text are simultaneously negative or simultaneously positive, thus constituting a binary classification task. Furthermore, we calculate the similarity between samples based on the average of the image and text modality embeddings, constructing a KNN graph ($K=5$).

\section{Baseline Details}
\label{appendix: baseline details}

\textbf{GCN}~\cite{gcn} employs a first-order approximation of spectral convolutions combined with a renormalization technique.
It performs message passing in a layer-wise manner, producing node embeddings that jointly reflect the local graph topology and node features, with computation cost scaling linearly with the number of edges.

\vspace{+0.1cm}
\noindent \textbf{GCNII}~\cite{gcnii} enhances standard GCNs by introducing initial residual connections and identity mapping.
    These additions implicitly simulate lazy random walks, mitigating over-smoothing and enabling deeper network architectures without sacrificing representation quality.

\vspace{+0.1cm}
\noindent \textbf{GAT}~\cite{gat} leverages self-attention on graph nodes to compute learnable weights for each neighbor during aggregation.
    This mechanism allows the network to differentiate the importance of neighboring nodes, thereby capturing diverse local structures without relying on a fixed Laplacian matrix.

\vspace{+0.1cm}
\noindent \textbf{GATv2}~\cite{gatv2} extends GAT by introducing a dynamic attention function that rearranges operations in the attention computation.
    This adjustment enhances the model’s expressive power, enabling it to better approximate complex functions and remain robust against noisy or irrelevant edges.
   
\vspace{+0.1cm}
\noindent \textbf{MMGCN}~\cite{mmgcn} introduces a multimodal framework for micro-video recommendation by modeling user preferences across visual, acoustic, and textual channels.
    It builds separate modality-specific bipartite graphs, captures high-order interactions within each, and fuses them via a structured integration layer, effectively reflecting user-item dynamics in each sensory modality.
    
\vspace{+0.1cm}
\noindent \textbf{MGAT}~\cite{mgat} applies gated attention over parallel multimodal interaction graphs for personalized recommendation.
    By adaptively weighting different modalities, it disentangles fine-grained user interests and filters out noisy or conflicting signals, enhancing preference modeling robustness.
    
\vspace{+0.1cm}
\noindent \textbf{MLaGA}~\cite{mlaga} enables LLMs to reason over MAGs through a two-stage alignment process.
    It first aligns visual and textual embeddings with graph structures using contrastive pre-training, and then performs instruction tuning to inject graph connectivity priors into the model’s generative reasoning.

\vspace{+0.1cm}
\noindent \textbf{LGMRec}~\cite{lgmrec} is a multimodal recommender that models both local and global user interests through graph learning. It separates collaborative and multimodal signals in user embeddings and addresses sparsity in local interest modeling. A local graph embedding module learns collaborative and modality-specific embeddings independently, while a global hypergraph module captures overall dependencies among users and items. Combining these decoupled local and global embeddings improves recommendation accuracy and robustness.

\vspace{+0.1cm}
\noindent \textbf{GraphGPT-O}~\cite{graphgpt_o} is a multimodal LLM for joint understanding and generation over MAGs.
    It addresses non-Euclidean dependencies and scalability by combining personalized PageRank sampling with a hierarchical alignment mechanism that links node- and graph-level representations to the LLM semantic space.
    
\vspace{+0.1cm}
\noindent \textbf{Graph4MM}~\cite{graph4mm} integrates multi-hop structural cues into attention via a hop-diffused strategy.
    A dedicated MM-QFormer performs cross-modal fusion, demonstrating that incorporating graph topology as an active interaction modality outperforms approaches treating it only as auxiliary input.
    
\vspace{+0.1cm}
\noindent \textbf{InstructG2I}~\cite{instructg2i} presents a graph-conditioned diffusion model for MAGs.
    It constructs contextual prompts through semantic neighbor sampling and encodes them with a Graph-QFormer, offering controllable generation via a graph-aware classifier-free guidance mechanism.
    
\vspace{+0.1cm}
\noindent \textbf{DMGC}~\cite{dmgc} models hybrid neighborhood patterns by separating cross-modality homophily-enhanced components from modality-specific heterophily-aware ones.
    Its dual-frequency fusion mechanism acts as coupled low- and high-pass filters, capturing both intra-class smoothness and inter-class distinctions simultaneously.
    
\vspace{+0.1cm}
\noindent \textbf{DGF}~\cite{dgf} proposes cross-contrastive clustering with dual graph filtering to denoise MAG features.
    A tri-cross contrastive objective across modalities, neighborhood structures, and semantic communities enables learning of discriminative and robust clustering representations.
    
\vspace{+0.1cm}
\noindent \textbf{MIG-GT}~\cite{mig_gt} uses modality-independent GNNs with adaptive receptive fields to handle diverse propagation needs and noise levels.
    It complements local aggregation with a sampling-based global transformer to capture long-range semantic dependencies that conventional message passing may miss.
    
\vspace{+0.1cm}
\noindent \textbf{NTSFormer}~\cite{ntsformer} introduces a self-teaching graph transformer for cold-start node classification.
    A stochastic attention mask supervises student predictions (self-information only) with teacher predictions (neighbor-aware), ensuring stable performance when structural links or attributes are partially missing.
    
\vspace{+0.1cm}
\noindent \textbf{UniGraph2}~\cite{unigraph2} is a cross-domain graph foundation model integrating multiple modalities into a shared embedding space.
    It combines frozen pre-trained encoders with a mixture-of-experts module for alignment, followed by a universal GNN for structure-aware aggregation, producing transferable representations for diverse downstream tasks.
\section{Limitations and Broader Impact.}
\label{appendix: lim}

\textbf{Limitations.} Our framework primarily focuses on modality-topology co-evolution grounded in graph signal processing theory. Integration of TMTE with existing advanced MGL methods is not explored and may require additional adaptation. Moreover, while the anchor-based mechanism of TMTE improves scalability, efficiency on ultra-large multimodal relational datasets remains to be validated, benefiting from distributed learning (e.g., federated graph learning~\cite{fedbook, fedgm}) in future work.

\vspace{+0.1cm}
\noindent \textbf{Broader Impacts.} TMTE offers a generalizable paradigm for learning task-specific graph structures in multimodal settings. By reducing reliance on manually curated or task-agnostic topologies, TMTE may enable more robust and adaptive multimodal graph learning in practical applications, potentially promoting fairer and more efficient MAG-based AI systems in domains such as healthcare, education, and social services.

\end{document}